\documentclass{article}

\usepackage[colorlinks=true,breaklinks=true,bookmarks=true,urlcolor=blue,citecolor=blue,linkcolor=blue,bookmarksopen=false,draft=false]{hyperref}
\usepackage{appendix}
\usepackage{color}
\usepackage{latexsym}
\usepackage{dsfont}
\usepackage{float}
\usepackage{algorithm,algorithmic}
\usepackage{graphicx}
\usepackage{amsmath, amsfonts,amssymb,theorem,euscript,array,enumerate,amsfonts,mathrsfs}
\usepackage{subcaption}
\usepackage{ulem}
\usepackage{bbm}
\usepackage[T1]{fontenc}
\usepackage[latin1]{inputenc}
\usepackage{hhline}
\usepackage{verbatim}
\usepackage{eurosym}
\newcommand{\DECAC}{\textsc{LTDE-Neural-AC}}

\usepackage{endnotes}
\let\footnote=\endnote
\let\enotesize=\normalsize
\def\notesname{Endnotes}
\def\makeenmark{$^{\theenmark}$}
\def\enoteformat{\rightskip0pt\leftskip0pt\parindent=1.75em
\leavevmode\llap{\theenmark.\enskip}}

\usepackage{natbib}
\def\bibfont{\small}%
\bibpunct[, ]{[}{]}{,}{n}{}{,}%

\numberwithin{equation}{section}

\newtheorem{Theorem}{Theorem}[section]
\newtheorem{Definition}[Theorem]{Definition}
\newtheorem{Proposition}[Theorem]{Proposition}
\newtheorem{Assumption}[Theorem]{Assumption}
\newtheorem{Lemma}[Theorem]{Lemma}

\newtheorem{Remark}[Theorem]{Remark}

\renewcommand{\theequation}{\thesection.\arabic{equation}}

\newenvironment{proof}[1][{\it Proof.}]{\begin{trivlist}
\item[\hskip \labelsep {\bfseries #1}]}{ \hfill
$\Box$\end{trivlist}\vskip -0.2 cm}

\theorembodyfont{\upshape}

\newcommand{\R}{\mathbb{R}}

\newcommand{\E}{\mathcal{E}}

\def\esssup_#1{\underset{#1}{\mathrm{ess\,sup\, }}}
\def\essinf_#1{\underset{#1}{\mathrm{ess\,inf\, }}}
\def\argmax_#1{\underset{#1}{\mathrm{arg\,max\, }}}
\def\argmin_#1{\underset{#1}{\mathrm{arg\,min\, }}}

\def \reff#1{{\rm(\ref{#1})}}

\def \Prod{\displaystyle\prod}

\def\b1{\bf 1}

\def \A{\mathbb{A}}

\def \N{\mathbb{N}}
\def \R{\mathbb{R}}

\def \E{\mathbb{E}}

\def \P{\mathbb{P}}

\def \A{{\cal A}}
\def \Ac{{\cal A}}
\def \Bc{{\cal B}}

\def \Ec{{\cal E}}
\def \Fc{{\cal F}}

\def \Hc{{\cal H}}

\def \Pc{{\cal P}}

\def \Oc{{\cal O}}
\def \Nc{{\cal N}}

\def \Sc{{\cal S}}
\def \Tc{{\cal T}}

\def \Yc{{\cal Y}}

\def \Xc{{\cal X}}

\def \reff#1{{\rm(\ref{#1})}}
\def \beqs{\begin{eqnarray*}}
\def \enqs{\end{eqnarray*}}
\def \beq{\begin{eqnarray}}
\def \enq{\end{eqnarray}}

\begin{document}

\title{Mean-Field Multi-Agent Reinforcement Learning: A Decentralized Network Approach}

\author{Haotian Gu
\thanks{Department of Mathematics, University of California, Berkeley, USA. \textbf{Email:} haotian$\_$gu@berkeley.edu }
\and
Xin Guo
\thanks{Department of Industrial Engineering \& Operations Research, University of California, Berkeley, USA. \textbf{Email:} xinguo@berkeley.edu}
\and
Xiaoli Wei
\thanks{Tsinghua-Berkeley Shenzhen Institute, Shenzhen, China. \textbf{Email:} xiaoli\_wei@sz.tsinghua.edu}
\and
Renyuan Xu
\thanks{Industrial  \& Systems Engineering, University of Southern California, Los Angeles, USA. \textbf{Email:} renyuanx@usc.edu}
}

\maketitle

\begin{abstract}
{
One of the challenges for multi-agent reinforcement learning (MARL) is designing efficient learning algorithms for a large system in which each agent has only limited or partial information of the entire system. 
While exciting progress has been made to analyze decentralized MARL with the {\it network of agents} for social networks and team video games, little is known theoretically for decentralized MARL with the {\it network of states} for modeling self-driving vehicles, ride-sharing, and data and traffic routing.

This paper proposes a framework of {\it localized training and decentralized execution} to study MARL with {\it network of states}.
Localized training means that agents only need to collect local information in their neighboring states during the training phase; decentralized execution implies that agents can execute afterwards the learned decentralized policies, which depend only on agents' current states. 

The theoretical analysis consists of three key components: the first is the 
reformulation of the MARL system as a networked Markov decision process with teams of agents, enabling updating the associated team Q-function in  a localized fashion;
the second is the Bellman equation for the value function and the appropriate Q-function on the probability measure space; and the third is the exponential decay property of the team Q-function, facilitating its approximation with efficient sample efficiency and controllable error.

The theoretical analysis paves the way for a new  algorithm {\DECAC}, where the actor-critic approach with over-parameterized neural networks is proposed.
The convergence and sample complexity is established and shown to be scalable with respect to the sizes of both agents and states.
To the best of our knowledge, this is the first neural network based MARL algorithm with network structure and provably convergence guarantee.
}
\end{abstract}

{\bf Keywords:}    Multi-Agent Reinforcement Learning, Mean-field Cooperative Games, Neural Network Approximation


\section{Introduction}
Multi-agent reinforcement learning (MARL) has achieved substantial successes in a broad range of cooperative games and their applications, including coordination of robot swarms (\citet{huttenrauch2017guided}), self-driving vehicles (\citet{SSS2016,cabannes2021solving}), real-time bidding games (\citet{JSLGWZ2018}), ride-sharing (\citet{LQJYWWWY2019}), power management (\citet{zhou2021robust}) and traffic routing (\citet{EAA2013}).
One of the challenges for the  development of MARL is designing  efficient learning algorithms for a large system, in which each individual agent has only limited or partial information of the entire system. In such a system, it is necessary to design algorithms to  learn policies of the decentralized type, i.e., policies that depend only on the {\it local} information of each agent.

In a simulated or laboratory setting, decentralized policies may be learned in a
centralized fashion. It is to train a central controller to dictate the actions of all agents. Such paradigm of {\it centralized training with decentralized execution} has achieved significant  empirical successes, especially with the computational power of deep neural networks (\citet{lowe2017multi,foerster2018counterfactual,chen2018communication,rashid2018qmix,yang2020multi,vadori2020calibration}). Such a training approach, however, suffers from the curse of dimensionality as the computational complexity grows exponentially with the number of agents (\citet{zhang2019multi});
it also requires extensive and costly communications between the central controller and all agents (\citet{rabbat2004distributed}). Moreover, policies derived from the centralized training stage may not  be robust  in the execution phase (\citet{zhang2020distributed}).
Most importantly, this approach has not been supported or analyzed theoretically. 

An alternative and  promising  paradigm is to take into consideration  the network structure of the system to train
decentralized policies. Compared with the centralized training approach, exploiting network structures  makes the training procedure more efficient as it allows the algorithm to be updated with parallel computing and reduces communication cost.

There are two distinct types of network structures. The first is the {\it network of agents}, often found in social networks {such as} Facebook {and Twitter}, as well as team video games {including StarCraft II}.
This network describes {\it interactions and relations among  heterogeneous agents}. For MARL systems with such network of agents, \citet{zhang2018fully}  establishes the asymptotic convergence of decentralized-actor-critic  algorithms which are scalable in agent actions. Similar ideas are extended to the  continuous space where deterministic policy gradient method (DPG) is used (\citet{zhang2018networked2}), with finite-sample analysis for such  framework established in the batch setting (\citet{zhang2021finite}). \citet{qu2020scalable} studies a network of agents where state and action interact in a local manner; by exploiting the network structure and the exponential decay property of the Q-function,  it proposes an actor-critic framework scalable in both actions and states. Similar framework is considered for the linear quadratic case with local policy gradients conducted with zero order optimization and parallel updating (\citet{li2019distributed}).

The second type of network, {\it the network of states}, has been frequently used for modeling self-driving vehicles, ride-sharing, and data and traffic routing. It focuses on the {\it state} of agents. Compared with the  network of agents which is {\it static} from agent's perspective (\citet{sunehag2017value}), the network of states is {\it stochastic}: neighboring agents of any given agent may change dynamically.
This type of network has been empirically studied  in various applications, including packet routing (\citet{you2020toward}), traffic routing (\citet{calderone2017markov}, \citet{gueriau2018samod}), resource allocations (\citet{cao2012overview}) and social economic systems (\citet{zheng2020ai}). However, there is no existing theoretical analysis for this type of decentralized MARL. Moreover, the dynamic nature of  agents' relationship makes it difficult to adopt existing methodology from the static network of agents. 
The goal of this paper is, therefore, to fill this void.

\paragraph{Our work.} This paper proposes and studies  multi-agent systems with network {structure} of {agent} states.
In this network, homogeneous agents can move from one state to any connecting state, and observe (realistically) only partial information of the entire system in an aggregated fashion.  To study this system, we propose a framework of {\it  localized training and decentralized execution} (LTDE). Localized training means that agents only need to collect local information in their neighboring states during the training phase;
decentralized execution implies that, agents can execute afterwards the learned decentralized policies which only require knowledge of  agents' current states.

The theoretical analysis consists of three key components. The first is to regroup these homogeneous agents according to their states and reformulate the MARL system as a networked Markov decision process with teams of agents. This reformulation leads to the decomposition of the Q-function and the value function according to the states, enabling the update of the consequent team Q-function in  a localized fashion. The second is to establish the Bellman equation for the value function and the appropriate Q-function on the probability measure space, by utilizing the homogeneity of agents. These functions are invariant with respect to the number of agents. The third is to explore the exponential decay property of the team Q-function, enabling its approximation with a truncated version of a much smaller dimension and yet with a controllable approximation error.

To design an efficient and scalable reinforcement learning algorithm for such framework, the actor-critic approach with over-parameterized neural networks is adopted. The neural networks, representing decentralized policies and localized Q-functions, are much smaller compared with the global one. The convergence and the sample complexity of the proposed algorithm is established and shown to be scalable with respect to the size of both agents and states.
To the best of our knowledge, this is the first neural network based MARL algorithm with network structure and provably convergence guarantee.

\paragraph{Our contribution.} Our work contributes to two lines of research.

The first one is for mean-field control with reinforcement learning, for which existing works  require that each agent have the full information of the population distribution (\citet{gu2020mean,carmona2019linear,carmona2019model,motte2019mean}) and yet in most applications agents only have access to partial or limited information ({\citet{yang2020q}}). We build a theoretical framework that incorporates network structures in the MARL framework, and provide computationally efficient algorithms where each agent only needs  local information of neighborhood states {to learn and to execute the policy}.  

Secondly, our work builds the  theoretical foundation for the practically popular scheme of centralized-training-decentralized-execution (CTDE) (\citet{lowe2017multi,rashid2018qmix,vadori2020calibration,yang2020multi}).
The CTDE framework is first proposed in \citet{lowe2017multi} to learn optimal policies in cooperative games with two steps: the first step is to train a global policy for the central controller, and the second step is to decompose the central policy (i.e., a large Q-table) into individual policies so that individual agent can apply the decomposed/decentralized policy after training. Despite the popularity of CTDE, however, there has been no theoretical study as to when the Q-table can be decomposed and when the truncation error can be controlled, except for a heuristic argument by \citet{lowe2017multi} for large $N$ with local observations.
Our paper analyzes for the first time with theoretical guarantee that 
applying  our algorithm to this CTDE paradigm yields a near-optimal sample complexity, when there is a network structure among agent states. Moreover, our algorithm, which is easier to scale-up, improves the centralized training step with a localized training. To differentiate our approach from the CTDE scheme, we call it localized-training-decentralized-execution (LTDE).

\paragraph{Notation.} For a set $\Xc$, denote $\mathbb{R}^{\Xc}=\{f:\Xc\to\mathbb{R}\}$ as the set of all real-valued functions on $\Xc$ . For each $f\in\mathbb{R}^{\Xc}$, define $\|f\|_\infty = \sup_{x \in \Xc} |f(x)|$ as the sup norm of $f$. In addition, when $\Xc$ is finite, denote $|\Xc|$ as the size of $\Xc$, and $\Pc(\Xc)$ as the set of all probability measures on $\Xc$: $\Pc(\Xc)=\{p:p(x)\geq 0, \sum_{x\in\Xc}p(x)=1\}$, which is equivalent to the probability simplex in $\mathbb{R}^{|\Xc|}$.
$[N]:=\{1,2,\cdots,N\}$. For any $\mu \in \Pc(\Xc)$ and a subset $\Yc\subset\Xc$, let $\mu(\Yc)$ denote the restriction of the vector $\mu$ on $\Yc$, and let $\Pc(\Yc)$ denote the set $\{\mu(\Yc):\mu\in\Pc(\Xc)\}$. For $x \in \R^d$, $d \in \N$, denote $\|x\|_2$ as the $L^2$-norm of $x$ and $\|x\|_\infty$ as the $L^\infty$-norm of $x$.

\section{Mean-field MARL with Local Dependency}\label{sec:set-up}
The focus of this paper is to study a cooperative multi-agent system with a network of agent states, 
which consists of nodes representing states of the agents and edges by which states are connected. In this system, every agent is only allowed to move from her present state to its connecting states.
Moreover, she is assumed to only observe (realistically)  {\it partial information} of the system on an aggregated level.  {Mean-field theory provides efficient approximations when agents only observe aggregated information, and has been applied in stochastic systems with large homogeneous agents such as financial markets (\citet{carmona2013mean,lacker2019mean,hu2021n,casgrain2020mean}), energy markets (\citet{germain2021level,aid2020entry}), and auction systems (\citet{iyer2014mean,guo2019learning}).}

\subsection{Review of  MARL}\label{sec:preliminary}
Let us first recall the cooperative MARL in an infinite time horizon, where there are $N$ agents whose policies are coordinated by a central controller. We assume that both the state space $\mathcal{S}$ and the action space $\mathcal{A}$ are finite.

At each step $t=0,1, \cdots,$ the state of agent $i$ $(= 1, 2, \cdots , N)$ is $s_t^{i} \in \mathcal{S}$ and she takes an action $a_t^{i} \in \mathcal{A}$.  Given the current state profile $\pmb{s}_t = (s_t^{1},\cdots,s_t^{N})\in \mathcal{S}^N$ and the current action profile $\pmb{a}_t = (a_t^{1},\cdots,a_t^{N})\in \mathcal{A}^N$ of $N$ agents, agent $i$ will receive a reward $ r^i({\pmb s}_t, {\pmb a}_t)$ and her state will change to $s_{t + 1}^{i}$ according to a transition probability function $P^i({\pmb s}_t, {\pmb a}_t)$. A Markovian game further restricts the admissible policy for agent $i$ to be of the form $a_t^{i} \sim \pi_t^i(\pmb{s}_t)$. That is, $\pi_t^i:  \mathcal{S}^N \rightarrow \mathcal{P}(\mathcal{A})$ maps each state profile $\pmb{s}\in \mathcal{S}^N$ to a randomized action, with $\mathcal{P}(\mathcal{A} )$ the space of all probability measures on space $\mathcal{A}$.

In this cooperative MARL framework, the central controller is to maximize the expected discounted accumulated reward averaged over all agents. That is  to find
\begin{eqnarray} \label{eq:preliminar_V}
V(\pmb{s})=\sup_{{\pmb \pi}}\frac{1}{N} \sum_{i=1}^N v^i({\pmb s}, {\pmb \pi}),
\end{eqnarray}
\vspace{-0.3cm}
where
\vspace{-0.2cm}
\begin{eqnarray}
v^i({\pmb s}, {\pmb \pi}) = \mathbb{E} \biggl[\sum_{t=0}^\infty \gamma^t  r^i({\pmb s}_t, {\pmb a}_t) \big| {\pmb s}_0 = {\pmb s}\biggl]
\end{eqnarray}
is the accumulated reward for agent $i$, given the initial state profile $\pmb{s}_0 = \pmb{s}$ and policy ${\pmb \pi} = \{{\pmb \pi}_t\}_{t = 0}^\infty$ with ${\pmb \pi}_t=(\pi_t^1, \ldots, \pi_t^N)$. Here
$\gamma$ $\in$ $(0, 1)$ is a discount factor, {$a_t^{i} \sim \pi_t^i(\pmb{s}_t)$, and $s_{t+1}^{i}\sim P^i({\pmb s_t}, {\pmb a_t})$}.

The corresponding Bellman equation for the value function \eqref{eq:preliminar_V} is
\begin{eqnarray} \label{equ:classicalV_fh}
V(\pmb{s}) = \sup_{\pmb{a} \in \Ac^N} \left\lbrace \E\left[\frac{1}{N}\sum_{i=1}^Nr^i(\pmb{s}, \pmb{a})\right] + \gamma \E_{\pmb{s}^{\prime} \sim \pmb{P}(\pmb{s}, \pmb{a})}\left[V(\pmb{s}^{\prime})\right] \right\rbrace,
\end{eqnarray}
with the population transition kernel $\pmb{P} = (P^1,\cdots,P^N)$.
The value function can be written as
$$V(\pmb{s}) = \sup_{\pmb{a} \in \Ac^N} Q(\pmb{s}, \pmb{a}),$$
in which the Q-function is defined as
\begin{eqnarray} \label{defsingleQ}
Q(\pmb{s}, \pmb{a}) = \E\left[\frac{1}{N}\sum_{i=1}^Nr^i(\pmb{s}, \pmb{a})\right] + \gamma \E_{\pmb{s}' \sim \pmb{P}(\pmb{s}, \pmb{a})}[V(\pmb{s}^{\prime})],
\end{eqnarray}
consisting of the expected reward from taking action $\pmb{a}$ at state $\pmb{s}$ and then following the optimal policy thereafter.
The Bellman equation for the Q-function, defined from $\Sc^N \times \A^N$ to $\R$, is given by
\begin{eqnarray} \label{equ: singleQ}
Q(\pmb{s}, \pmb{a}) = \E\left[\frac{1}{N}\sum_{i=1}^Nr^i(\pmb{s}, \pmb{a})\right] + \gamma \E_{\pmb{s}' \sim \pmb{P}(\pmb{s}, \pmb{a})}\left[\sup_{\pmb{a}' \in \Ac^N} Q(\pmb{s}', \pmb{a}')\right].
\end{eqnarray}
One can thus retrieve the
optimal (stationary) control $\pi^*(\pmb{s}, \pmb{a})$ (if it exists) from  $Q(\pmb{s}, \pmb{a})$, with $\pi^*(\pmb{s}) \in \arg\max_{\pmb{a} \in \Ac^N} Q(\pmb{s}, \pmb{a})$.


\subsection{Mean-field MARL with Local Dependency}\label{subsec:local_setting}
In this system, there are $N$ agents who share a finite state space $\Sc$ and take actions from a finite action space $\mathcal{A}$. Moreover, there is a network on the state space $\Sc$ associated with an underlying undirected graph $(\Sc, \Ec)$, where $\Ec\subset\Sc\times\Sc$ is the set of edges. The distance between two nodes is defined as the number of edges in a shortest path. For a given $s \in \Sc$,  $\Nc_s^1$ denotes the nearest neighbor of $s$, which consists of all nodes connected to $s$ by an edge and includes $s$ itself; and $\Nc^k_s$ denotes the $k$-hop neighborhood of $s$, which consists of all nodes whose distance to $s $ is less than or equal to $k$, including $s$ itself. For simplicity, we use $\Nc_s: = \Nc^1_s$.
From agent $i$'s perspective, agents in her neighborhood $\Nc_{s^i_t}$ change stochastically over time.

To facilitate mean-field approximation to this system, assume throughout the paper that the agents are homogeneous and indistinguishable.
In particular, at each step $t=0,1, \cdots,$ if agent $i$ at state $s^i_t \in \Sc$  takes an action $a^i_t \in \mathcal{A}$, then she will receive a stochastic reward which is uniformly upper bounded by $r_{\max}$ such that
\vspace{-0.2cm}
\begin{eqnarray}\label{eq:N_agent_reward_heter}
{r^i(\pmb{s}_t,\pmb{a}_t) :=} r\Big(s^i_t, \,\,\mu_t(\Nc_{s^i_t}), \,\,a^i_t\Big)\le r_{\max}, \quad i \in [N];
\end{eqnarray}
and her state will change to a neighboring state $s^i_{t + 1}$ $\in \Nc_{s_t^i}$ according to a transition probability such that
\begin{eqnarray}\label{eq:N_agent_transition_heter}
s^i_{t + 1}\sim {P^i(\pmb{s}_t,\pmb{a}_t) :=}{P}\left.\Big(\cdot\,\right\vert\,s^i_t, \,\,\mu_t(\Nc_{s^i_t}), \,\,a^i_t\Big), \quad i \in [N],
\end{eqnarray}
where $\mu_t(\cdot) = \frac{\sum_{i=1}^N{\bf 1}(s_t^i=\cdot)}{N}$ $\in \mathcal{P}^N(\Sc)$ $:=\left\{\mu\in\Pc(\Sc):\mu(s)\in\left\{0,\frac{1}{N},\frac{2}{N},\cdots, \frac{N-1}{N}, 1\right\}\text{ for all }s\in\Sc\right\}$ is the  empirical state distribution of $N$ agents at time $t$, with $N \cdot \mu_t(s)$ the number of agents in state $s$ at time $t$, and $\mu_t(\Nc_{s^i_t})$  the restriction of $\mu_t$ on the 1-hop neighbor of $s_t^i$. 

\eqref{eq:N_agent_reward_heter}-\eqref{eq:N_agent_transition_heter} indicate that the reward and the transition probability of agent $i$ at time $t$ depend on both her individual information $(a_t^i,s_t^i)$ and the mean-field of her 1-hop neighborhood $\mu_t(\Nc_{s^i_t})$, in an aggregated yet localized format: {\it aggregated} or {\it mean-field} meaning that agent $i$ depends on other agents only through the empirical state distribution; {\it localized} meaning that agent $i$ depends on the  mean-field information of her 1-hop neighborhood.
Intuitive examples of such  a setting include traffic-routing, package delivery, data routing, resource allocations, distributed control of autonomous vehicles and social economic systems.

\paragraph{Policies with partial information.} To incorporate the element of {\it partial or limited information} into this mean-field MARL system, consider the following {\it individual-decentralized {policies}}
\begin{eqnarray}\label{eq:N_agent_policy}
a^i_t \sim  {\pi^i(\pmb{s}_t)}: = \pi\left(s^i_t,\mu_t(s^i_t)\right)\in \mathcal{P}(\Ac), \quad i \in [N],
\end{eqnarray}
and denote $\mathfrak{u}$ as the admissible policy set of all such policies.

Note that for a given mean-field information $\mu_t$, $\pi(\cdot, \mu_t(\cdot)):\Sc \rightarrow \mathcal{P}(\mathcal{A})$  maps the agent state to a randomized action. That is, the policy of each agent is executed in a decentralized manner and assumes that each agent only has access to the population information in her own state.
This is more realistic than centralized {policies} which assume full access to the state information of all agents.

\paragraph{Value function and Q-function.} The goal for this mean-field  MARL is to maximize the expected discounted accumulated reward averaged over all agents, i.e.,
\begin{equation}\label{eq:N_agent_reward}
V(\mu):=\sup_{\pi{
\in \mathfrak{u}}}V^{\pi}(\mu) = \sup_{\pi{
\in \mathfrak{u}}}\frac{1}{N}\sum_{i=1}^N\mathbb{E} \left.\left[\sum_{t=0}^\infty \gamma^t {r\big(s^i_t, \,\,\mu_t(\Nc_{s^i_t}), \,\,a^i_t\big)}\,\,\right|\,\,\mu_0=\mu\right],\tag{MF-MARL}
\end{equation}
subject to \eqref{eq:N_agent_reward_heter}-\eqref{eq:N_agent_policy}  with a discount factor $\gamma$ $\in$ $(0, 1)$. 

The mean-field assumption leads to the following definition of the corresponding Q-function for (MF-MARL) on the measure space:
\vspace{-0.4cm}
\begin{eqnarray}
Q({\mu},{h}):&=&\underbrace{\mathbb{E}\left.\Big[
\sum_{i=1}^N \frac{1}{N}{r}(s_0^i,\mu(\Nc_{s^i_0}),a^i_0)
\right|\pmb{s}_0,\pmb{a}_0\Big]}_{\text{Expected reward of taking $\pmb{a}_0=(a_0^1,\cdots,a_0^N)$}}\nonumber\\
& & +\,\, \underbrace{\mathbb{E}_{s^i_{1}\sim {P}\left.\big(\cdot\,\right\vert\,s^i_0, \,\,\mu(\Nc_{s^i_0}), \,\,a^i_0\big)} \left.\left[\sum_{t = 1}^\infty \gamma^t \sum_{i=1}^N \frac{1}{N} r(s^i_t,\mu_t(\Nc_{s^i_t}), a^i_t)\right|a_t^i \sim \pi_t^\star\right]}_{\text{Expected reward of playing optimally thereafter
$a_t^i\sim \pi_t^\star$}}, \label{eq:Q_mf}
\end{eqnarray}
where $\mu(\cdot) =  \frac{\sum_{i=1}^N{\bf 1}(s_0^i=\cdot)}{N}$ is the initial empirical state distribution
and $h(s)(a) = \frac{\sum_{i=1}^N {\bf 1}(s_0^i={s}, a_0^i={a})}{\sum_{i=1}^N {\bf 1}(s_0^i={s})}$ is a ``decentralized'' policy representing the proportion of agents in state $s$ that takes action $a$.
Specifically, given $\mu\in\Pc^N(\Sc)$,  $s\in\Sc$, and the $N\cdot\mu(s)$ agents in state $s$,
\begin{eqnarray*}\label{eq:h_space}
h(s)\in\Pc^{N\cdot\mu(s)}(\Ac): = \left\{\varsigma \in \Pc(\A): \varsigma(a) \in \{ 0,\frac{1}{N\cdot\mu(s)},\cdots, \frac{N\cdot\mu(s)-1}{N\cdot\mu(s)}\} \text{ for all } a\in\Ac \right\} \subset \Pc(\A),
\end{eqnarray*}
where $\varsigma$ in $\Pc^{N\cdot\mu(s)}(\Ac)$ is an empirical action distribution of $N \cdot \mu(s)$ agents in state $s$, and $\varsigma(a)$ is the proportion of agents taking action $a\in\Ac$ among all $N \cdot \mu(s)$ agents in state $s$. 
Furthermore, for a given $s\in\Sc$, denote  $\Pc^{N\cdot\mu(s)}(\Ac)$ the set of all admissible ``decentralized'' policies $h(s)(\cdot)$; and for a
given $\mu \in \Pc^N(\Sc)$, denote the product of $\Pc^{N \cdot \mu(s)}(\Ac)$ over all states by $\Hc^N(\mu): = \{h: h(s) \in \Pc^{N \cdot \mu(s)}(\A)\; \forall \; s \in \Sc\}$. 
Here $\Hc^N(\mu)$ depends on $\mu$ and is a subset of $\Hc =\{h: \Sc \to \Pc(\A)\}$.

Note that $Q({\mu},h)$ defined in \eqref{eq:Q_mf} is invariant with respect to the order of the elements in $\pmb{s}_0$ and $\pmb{a}_0$. More critically, the input dimension of the Q-function defined in \eqref{eq:Q_mf} is {\it independent from the number of agents} in the system, hence is easier to scale up in a large population regime. This differs from the 
the input dimension of the Q-function in  \eqref{defsingleQ}, which  grows exponentially with respect to the number of agents, the main culprit of the curse of dimensionality for MARL algorithms.

\section{Analysis of MF-MARL with Local Dependency}
The theoretical study of this mean-field MARL with local dependency (Section \ref{subsec:local_setting}) consists of three key components, which are crucial for subsequent algorithm design and convergence analysis: the first  is the reformulation of the MARL system as a networked Markov decision process with teams of agents. This reformulation leads to the decomposition of the Q-function and the value function according to states, facilitating  updating the consequent team Q-function in a localized fashion (Section \ref{subsec:regroup}); the second is the Bellman equation for the value function and the  Q-function on the probability measure space (Section \ref{subsec:bellman});
the third is the exponential decay property of the team Q-function, enabling its approximation with a truncated version of a much smaller dimension and yet with a controllable  approximation error (Section \ref{subsec:exponential_decay}).


\subsection{Markov Decision Process (MDP) on Network of States}\label{subsec:regroup}
This section shows that the mean-field MARL (\ref{eq:N_agent_reward_heter})-(\ref{eq:N_agent_policy}) 
can be reformulated in an MDP framework by exploiting the  network structure of states.
This reformulation
leads to the decomposition of the Q-function, facilitating more computationally efficient updates.

The key idea is to utilize the homogeneity of the agents in the problem set-up  and  to regroup these $N$ agents according to their states.
This regrouping translates (MF-MARL) with $N$ agents into a networked MDP with $|\Sc|$ agents teams,   indexed  by their states.

\begin{figure}[!ht]
    \centering
    \includegraphics[width=0.35\textwidth]{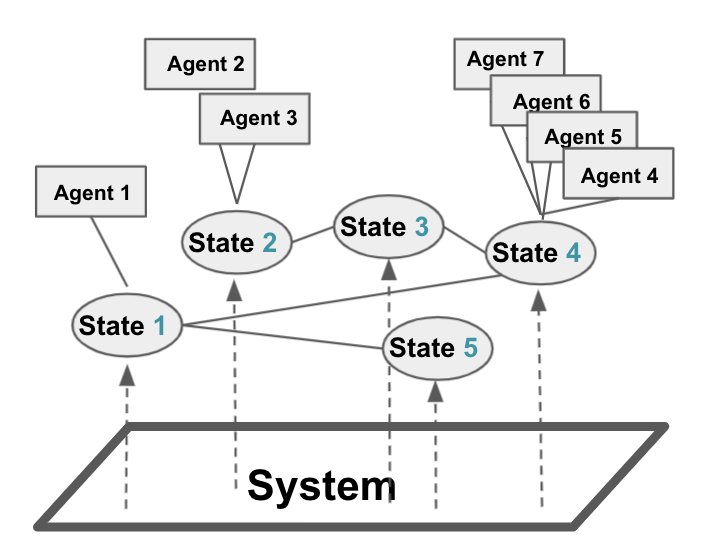}
     \includegraphics[width=0.35\textwidth]{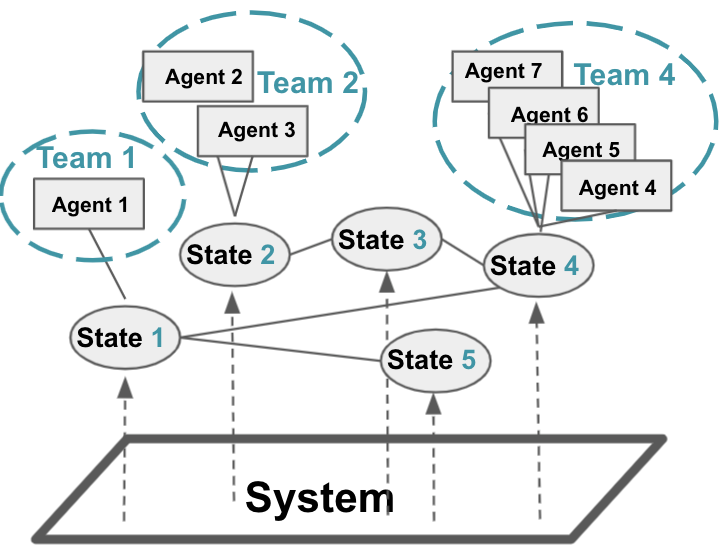}
    \caption{{\bf Left:} MF-MARL problem \eqref{eq:N_agent_reward_heter}-\eqref{eq:N_agent_policy}. {\bf Right:} Reformulation of team game \eqref{eq:lifted_policy_s}-\eqref{eqn:r_loc}.}
    \label{fig:state_net}
\end{figure}

To see how the policy, the reward function, and the dynamics in this networked Markov decision process are induced by the regrouping approach,  recall that there are  $N \cdot \mu(s)$ agents in state $s$, each agent $i$ in state $s$ will independently choose action $a_i\sim\pi(s, \mu(s))$ according to the individual-decentralized policy $\pi(s, \mu(s))\in\Pc(\Ac)$ in \eqref{eq:N_agent_policy}. 
Therefore the empirical action distribution of $\{a_1,\cdots,a_{N \cdot \mu(s)}\}$ is a random variable taking values from $\Pc^{N\cdot\mu(s)}(\Ac)$,  the set of empirical action distributions with $N \cdot \mu(s)$ agents.
Moreover, for any $h(s)\in\Pc^{N\cdot\mu(s)}(\Ac)$, we have
\begin{eqnarray}\label{eqn:equivalent_pi_Pi}
& &\P\left(h(s) \text{ is the empirical action distribution of } \{a_1,\cdots,a_{N \cdot \mu(s)}\}, a_i\overset{i.i.d}{\sim}\pi(s, \mu(s))\right) \nonumber\\
&=& \P\left(\text{for each } a\in\Ac, a \text{ appears } N \cdot \mu(s)h(s)(a) \text{ times in } \{a_1,\cdots,a_{N \cdot \mu(s)}\}, a_i\overset{i.i.d}{\sim}\pi(s, \mu(s))\right) \nonumber\\
&=& \frac{(N \cdot \mu(s))!}{\Prod_{a\in \Ac}(N \cdot \mu(s)h(s)(a))!} \Prod_{a\in \mathcal{A}}\Big(\pi(s,\mu(s))(a)\Big)^{N \cdot \mu(s)h(s)(a)}.
\end{eqnarray}
Here $h(s)(a)$ denotes the proportion of agents taking action $a$ among all agents in state $s$, with  last equality derived from the  multinomial distribution with parameters $N\cdot\mu(s)$ and $\pi(s,\mu(s))$.

Now, clearly  each individual-decentralized policy $\pi(s, \mu(s))\in\Pc(\Ac)$ in \eqref{eq:N_agent_policy}  induces a \textit{team-decentralized policy}
of the following form:
\begin{eqnarray}\label{eq:lifted_policy_s}
\Pi_s(h(s)\mid\mu(s)) = \frac{(N \cdot \mu(s))!}{\Prod_{a\in \Ac}(N \cdot \mu(s)h(s)(a))!} \Prod_{a\in \mathcal{A}}\Big(\pi(s,\mu(s))(a)\Big)^{N \cdot \mu(s)h(s)(a)},
\end{eqnarray}
where $h(s)\in\Pc^{N\cdot\mu(s)}(\Ac)$. 
Conversely, given a team-decentralized policy $\Pi_s(\,\cdot\mid\mu(s))$,  one can recover the individual-decentralized policy $\pi(s, \mu(s))$ by choosing appropriate $h(s)\in\Pc^{N\cdot\mu(s)}(\Ac)$ and querying the value of $\Pi_s(h(s)\mid\mu(s))$:  let $h_i(s) = \delta_{a_i}$ be the Dirac measure with $a_i \in \mathcal{A}$, which is an action distribution such that all agents in state $s$ take action $a_i$. By \eqref{eq:lifted_policy_s}, $\Pi_s(h_i(s)\mid\mu(s)) =\left(\pi(s, \mu(s))(a_i)\right)^{N \cdot \mu(s)}$, implying $\pi(s, \mu(s))(a_i) = \left(\Pi(h_i(s)|\mu(s)\right)^{\frac{1}{N \cdot \mu(s)}}$. 


Next, given $\mu\in\Pc^N(\Sc)$ and $h\in\Hc^N(\mu)=\{h:h(s)\in\Pc^{N\cdot\mu(s)}(\Ac), \forall s\in \Sc\}$, the set of empirical action distributions on every state, if we define
\begin{equation}\label{eqn:big_pi}
    \Pi(h\mid\mu):=\prod_{s\in\Sc}\Pi_s(h(s)\mid\mu(s)),
\end{equation}
then $\mathfrak{u}$,  the admissible policy set of individual-decentralized policies in the form of \eqref{eq:N_agent_policy}, is now replaced by $\mathfrak{U}$, the set of all team-decentralized policies $\Pi$ induced from $\pi\in\mathfrak{u}$ through \eqref{eq:lifted_policy_s} and \eqref{eqn:big_pi}. In addition, denote the set of all state-action distribution pairs as
\begin{eqnarray}
\Xi:=\cup_{\mu\in\Pc^N(\Sc)}\{\zeta=(\mu,h):h\in\Hc^N(\mu)\},
\end{eqnarray}

Moreover, from  the team perspective, the transition probability in \eqref{eq:N_agent_transition_heter} can  be viewed as a Markov process of $\mu_t$ and $h_t \in \mathcal{H}^N(\mu_t)$ with
an  induced transition probability $\mathbf{P}^N$ from \eqref{eq:N_agent_transition_heter} such that
\begin{equation}\label{eqn:P_N}
    \mu_{t+1}\sim \mathbf{P}^N(\cdot\,|\,\mu_t,h_t).
\end{equation} 
It is easy to verify that  for a given  state $s\in\Sc$, $\mu_{t+1}(s)$ only depends on $\mu_t(\Nc^2_s)$, the empirical  distribution in the 2-hop neighborhood of $s$, and $h_t(\Nc_{s})$.

Finally, given $\mu(\mathcal{N}_s) \in \Pc^N (\mathcal{N}_s)$, an empirical distribution restricted to the 1-hop neighborhood of $s$,  one can  define 
a {\it  localized team reward function for team $s$} from $\Pc^{N\cdot\mu(s)}(\Ac)$ to $\R$ as
\begin{equation}\label{eqn:r_loc}
    r_{s}(\mu(\mathcal{N}_s),h(s)) = \sum_{a\in\Ac}{r}(s,\mu(\Nc_{s}),a)h(s)(a),  
\end{equation}
 which depends on the state $s$ and its 1-hop neighborhood; and define the maximal expected discounted accumulative localized team rewards over all {\it teams} as
\begin{eqnarray}\label{eq:N_agent_reward_reformu}
\widetilde V(\mu):=\sup_{\Pi{
\in \mathfrak{U}}}\widetilde V^{\Pi}(\mu) =  \sup_{\Pi{
\in \mathfrak{U}}}  \mathbb{E}\biggl[ \sum_{t = 0}^\infty \sum_{s \in \Sc}\gamma^t\, r_s(\mu_t(\Nc_{s}), h_t(s)) \left| \mu_0 = \mu \biggl]\right..
\end{eqnarray}
 
With all these key elements, one can establish the equivalence between maximizing the reward averaged over all {\it agents} in \eqref{eq:N_agent_reward} and maximizing the localized team reward summed over all {\it teams} in \eqref{eq:N_agent_reward_reformu}, and can thus reformulate the (MF-MARL) problem as  an equivalent MDP of \eqref{eq:lifted_policy_s}-\eqref{eq:N_agent_reward_reformu} with $|\Sc|$ teams, the latter denoted as {(MF-DEC-MARL)}.
(The proof is detailed in Appendix \ref{app:section23lemma}).  That is, 

\begin{Lemma}(Value function and Q-function decomposition) \label{equiv_V_tildeV} 
\begin{eqnarray}\label{eq:value_decompostion}
V(\mu) = \widetilde V(\mu) =  \sup_{\Pi{
\in \mathfrak{U}}} \sum_{s\in \mathcal{S}}\widetilde V_s^{\Pi}(\mu),
\end{eqnarray}
where $h_t\sim\Pi(\cdot\,|\,\mu_t)$, $\mu_{t+1}\sim \mathbf{P}^N(\cdot\,|\,\mu_t,h_t)$, and 
\begin{equation}
 \widetilde V_s^{\Pi}(\mu) = \mathbb{E}\biggl[ \sum_{t = 0}^\infty\gamma^t\, r_s(\mu_t(\Nc_{s}), h_t(s)) \left| \mu_0 = \mu \biggl]\right.   
\end{equation} is called the value function under policy $\Pi$ for team $s$. Similarly,
\begin{align}\label{eq:Q_decompostion}
Q^{\Pi}(\mu,h): &= \mathbb{E}\biggl[ \sum_{t = 0}^\infty\gamma^t \sum_{s \in \Sc} r_s(\mu_t(\Nc_{s}), h_t(s))\left| \mu_0 = \mu, h_0=h \biggl]\right.=\sum_{s\in \mathcal{S}}Q^{\Pi}_s(\mu,h),
\end{align}
where 
\begin{equation}\label{eqn:team_Q}
  Q^{\Pi}_s(\mu,h) = \mathbb{E}\biggl[ \sum_{t = 0}^\infty\gamma^t r_s(\mu_t(\Nc_{s}), h_t(s))\left| \mu_0 = \mu, h_0=h \biggl],\right.  
\end{equation} {is the Q-function under policy $\Pi$} for team $s$, called  {\it team-decentralized Q-function}.
\end{Lemma}


The decomposition for the Q-function in \eqref{eq:Q_decompostion} is one of the key elements to allow for approximation of  $Q_s^{\Pi}(\mu,h)$ by a truncated Q-function defined on a smaller space and  updated in a localized fashion; it is
useful for designing sample-efficient learning algorithms and for parallel computing,  as will be clear in the next Section \ref{subsec:exponential_decay}.

\subsection{Bellman equation for Q-function.}\label{subsec:bellman}
This section builds the second block for
reinforcement learning algorithms, the Bellman equation for Q-function. 
Indeed, the Bellman equation for $Q({\mu},{h})$ can be derived
following a similar argument in \citet{gu2019dynamic}, after establishing the dynamic programming principle  on an appropriate probability measure space.
\begin{Lemma}(Bellman Equation for Q-function)\label{lemma:bellman}  The Q-function defined in \eqref{eq:Q_mf} satisfies:
\begin{eqnarray}\label{eqn:bellman_fullQ}
Q({\mu},h)={\mathbb{E}\left.\left[
\sum_{i=1}^N \frac{1}{N}{r}(s_0^i,\mu(\Nc_{s^i_0}),a^i_0)
\right|\pmb{s}_0,\pmb{a}_0\right]}+ \gamma \mathbb{E}_{s^i_{1}\sim {P}\left.\big(\cdot\,\right\vert\,s^i_0, \,\,\mu(\Nc_{s^i_0}), \,\,a^i_0\big)} \left[\sup_{h^{\prime} \in \mathcal{H}^N(\mu_1)}Q\left(\mu_1,h^{\prime}\right)\right].\label{eq:Bellman_Q_mf}
\end{eqnarray}
{with $\mu_1(\cdot) = \frac{\sum_{i=1}^N{\bf 1}(s_1^i=\cdot)}{N}$ the empirical state distribution at time $1$.}
\end{Lemma}

Note that the Bellman equation \eqref{eqn:bellman_fullQ} is for the Q-function defined in \eqref{eq:Q_mf} for general mean-field MARL. In order to enable the {\it localized-training-decentralized-execution} for computational efficiency, one needs to consider the decomposition of Q-function \eqref{eq:Q_decompostion}  and the updating rule based on the team-decentralized Q-function \eqref{eqn:team_Q}.  The corresponding Bellman equation for the team-decentralized Q-function \eqref{eqn:team_Q} is:
\begin{Lemma}
Given a policy $\Pi\in\mathfrak{U}$, ${Q}^{\Pi}_s$ defined in \eqref{eqn:team_Q} is the unique solution to the Bellman equation ${Q}^{\Pi}_s=\mathcal{T}^{\Pi}_s{Q}^{\Pi}_s$, with $\mathcal{T}^{\Pi}_s$ the Bellman operator taking the form of
\begin{equation}\label{eqn:bellman_op}
    \mathcal{T}^{\Pi}_s {Q}^{\Pi}_s(\mu, h)=\mathbb{E}_{{\mu}'\sim {\mathbf P}^N(\cdot\,|\,\mu,h),\,h'\sim\Pi(\cdot\,|\,\mu)}\left[ r_s(\mu, h)+\gamma\cdot {Q}^{\Pi}_s({\mu}', h')\right], \forall(\mu, h) \in \Xi.
\end{equation}
\end{Lemma}
These Bellman equations are the basis for general Q-function-based  algorithms
in mean-field MARL.


\subsection{Exponential Decay of Q-function}\label{subsec:exponential_decay}
This section will  show that the team-decentralized Q-function $Q_s^{\Pi}(\mu,h)$ has an {\it exponential decay} property.
This  is another key element to enable an approximation to $Q_s^{\Pi}$ by
a {\it localized Q-function} $\widehat{Q}^{\Pi}_s(\mu(\Nc_s^k),h(\Nc_s^k))$, and to guarantee the scalability and sample efficiency of subsequent algorithm design.

To establish the 
exponential decay property of the Q-function
\eqref{eqn:team_Q}, first recall that $\Nc^k_s$ is the set of $k$-hop neighborhood of state $s$, and define $\Nc^{-k}_s=\Sc / \Nc^{k}_s$ as the set of states that are outside of $s$'th $k$-hop neighborhood. Next, rewrite any given  empirical state distribution $\mu\in\Pc^N(\Sc)$ as $\left(\mu(\Nc^{k}_s), \mu(\Nc^{-k}_s)\right)$, and similarly,  $h\in\Hc^N(\mu)$  as $\left(h(\Nc^{k}_s), h(\Nc^{-k}_s)\right)$. 
\begin{Definition}\label{def:exp_decay}
    The $Q^{\Pi}$ is said to have $(c, \rho)$-exponential decay property, if for any $s \in \Sc$ and any 
     $\Pi\in \mathfrak{U}$, $(\mu, h), (\mu', h') \in \Xi$ with $\mu(\Nc^k_s) = \mu'(\Nc^k_s)$ and $h(\Nc^{k}_s) = h'(\Nc^k_s)$
     $$\bigg|Q_{s}^{\Pi}\left(\mu(\Nc^{k}_s), \mu(\Nc^{-k}_s), h(\Nc^{k}_s), h(\Nc^{-k}_s)\right)-Q_{s}^{\Pi}\left(\mu(\Nc^{k}_s), {\mu}'(\Nc^{-k}_s), h(\Nc^{k}_s), h'(\Nc^{-k}_s)\right)\bigg| \leq c \rho^{k+1}.$$
\end{Definition}

Note that the exponential decay property is defined for the  team-decentralized Q-function $Q_{s}^{\Pi}$, instead of the centralized Q-function $Q^{\Pi}$.
The following Lemma provides a sufficient condition for the exponential decay property. Its proof is given in Appendix \ref{app:exponential_decay_lemma}.

\begin{Lemma}\label{lemma:exp_decay_Q}
 When the reward $r_s$ in \eqref{eqn:r_loc} is uniformly upper bounded by $r_{\text{max}}>0$, for any $s\in\Sc$, $Q_s^{\Pi}$ satisfies the $\left(\frac{r_{\text{max}}}{1-\gamma}, \sqrt{\gamma}\right)$-exponential decay property.
\end{Lemma}

 The exponential decay property implies that for a given  state $s\in\Sc$, the dependence of $Q_{s}^{\Pi}$ on other states decays quickly with respect to its distance from state $s$.
It motivates and enables the approximation of  $Q_{s}^{\Pi}(\mu, h)$ by a truncated function which only depends on $\mu(\Nc^{k}_s)$ and $h(\Nc^{k}_s)$,  especially when $k$ is large and $\rho$ is small.  
Specifically, consider the following class of {\it localized} Q-functions,
\begin{align}\label{eqn:trunc_Q}
    \widehat{Q}_{s}^{\Pi}\Big(\mu(\Nc^{k}_s), h(\Nc^{k}_s)\Big)=\sum_{\mu(\Nc^{-k}_s), h(\Nc^{-k}_s)} &\Bigg[w_{s}\Big(\mu(\Nc^{-k}_s), h(\Nc^{-k}_s); \mu(\Nc^{k}_s), h(\Nc^{k}_s)\Big)  \nonumber\\
     &\cdot Q_{s}^{\Pi}\Big(\mu(\Nc^{k}_s), \mu(\Nc^{-k}_s), h(\Nc^{k}_s), h(\Nc^{-k}_s)\Big)\Bigg],  \tag{Local Q-function}
\end{align}
where $w_{s}\left(\mu(\Nc^{-k}_s), h(\Nc^{-k}_s); \mu(\Nc^{k}_s), h(\Nc^{-k}_s)\right)$ are any non-negative weights of $$\sum_{\mu(\Nc^{-k}_s), h(\Nc^{-k}_s)} w_{s}\Big(\mu(\Nc^{-k}_s), h(\Nc^{-k}_s); \mu(\Nc^{k}_s), h(\Nc^{k}_s)\Big)=1$$ for any $\mu(\Nc^{k}_s)$ and $h(\Nc^{k}_s)$.

Then, direct computation yields the following proposition.
\begin{Proposition}
Let $\widehat{Q}_{s}^{\Pi}$ be any localized Q-function in the form of \eqref{eqn:trunc_Q}. Assume the $(c, \rho)$-exponential decay property in Definition \ref{def:exp_decay} holds, then for any $\mu\in\Pc^N(\Sc)$ and $h\in\Hc^N(\mu)$,
\begin{eqnarray}
\left|\widehat{Q}_{s}^{\Pi}\Big(\mu(\Nc^{k}_s), h(\Nc^{k}_s)\Big) - {Q}_{s}^{\Pi}(\mu, h) \right| \leq c\rho^{k+1}.\label{eq:exponentialQhat}
\end{eqnarray}
Moreover, \eqref{eq:exponentialQhat} holds independent of the weights in \eqref{eqn:trunc_Q}.
\end{Proposition}

Note that given a team-decentralized Q-function $Q_{s}^{\Pi}$, its localized version $\widehat{Q}_{s}^{\Pi}$ only takes $\mu(\Nc^{k}_s), h(\Nc^{k}_s)$ as inputs, and $\widehat{Q}_{s}^{\Pi}\Big(\mu(\Nc^{k}_s), h(\Nc^{k}_s)\Big)$ is defined as a weighted average of ${Q}_{s}^{\Pi}$ over all $(\mu,h)$-pairs which agree with $\Big(\mu(\Nc^{k}_s), h(\Nc^{k}_s)\Big)$ in the $k$-hop neighborhood of $s$. {Although} the localized Q-function $\widehat{Q}_{s}^{\Pi}$ may vary according to different choices of the weights, by the exponential decay property, every $\widehat{Q}_{s}^{\Pi}$  approximates ${Q}_{s}^{\Pi}$ with uniform error and requires a smaller dimension of input.

\begin{Remark}(Exponential Decay Property) In a discounted reward setting  \eqref{eq:preliminar_V}, the exponential decay property follows directly from the fact that the discount factor $\gamma\in(0,1)$ and the local dependency structure in \eqref{eq:lifted_policy_s}-\eqref{eq:N_agent_reward_reformu}.
For problems of finite-time or infinite horizons with ergodic reward functions, this property can  be established by imposing additional Lipschitz condition on the transition kernel. (See \citet{qu2020scalable}, Theorem 1 for network of heterogeneous agents and $\gamma=1$).

It is also worth pointing out that the exponential decay property has been extensively explored in random graphs  (e.g., \citet{Gamarnik2013, GGW2014}) and for analysis of network of agents in \citet{qu2020scalable} and \citet{lin2021multi}.
\end{Remark}

\section{Algorithm Design}
\label{sec:algorithm}
The three key analytical components for problem (MF-DEC-MARL) in previous sections pave the way for designing efficient learning algorithms. In this section, we propose and analyze  a decentralized neural actor-critic algorithm, called {\DECAC}. 

Our focus is the localized Q-function $\widehat{Q}^{\Pi}_s(\mu(\Nc_s^k),h(\Nc_s^k))$, the approximation to $Q_s^{\Pi}$  with a smaller input dimension. First, this localized Q-function $\widehat{Q}_s^{\Pi}$ and the team-decentralized policy $\Pi_s$ will be parameterized by two-layer neural networks with parameters $\omega_s$ and $\theta_s$ respectively (Section \ref{sec:neurals}). 
Next, these neural network parameters $\theta=\{\theta_s\}_{s\in\Sc}$ and $\omega=\{\omega_s\}_{s\in\Sc}$ are updated via an actor-critic algorithm in a {\it localized fashion} (Section \ref{sec:actor-critic}): the critic aims to find a proper estimate for the localized Q-function under a fixed policy (parameterized by $\theta$), while the actor computes the policy gradient based on the localized Q-function, and updates $\theta$ by a gradient step. 

These networks are updated locally requiring only information of the neighborhood states during the training phase; afterwards agents in the system will execute these learned {\it decentralized policies}  which requires only information of the agent's current state. This {\it localized training and decentralized execution} enables efficient parallel computing especially for a large shared state space.

Moreover, over-parameterization of neural networks avoids issues of nonconvexity and divergence  associated with the neural network approach, and ensures the global convergence of our proposed {\DECAC} algorithm.

\subsection{Basic Set-up}
\paragraph{Policy parameterization.}  
To start, let us assume that at state $s$ the {\it team-decentralized policy} $\Pi^{\theta_s}_s$ is parameterized by $\theta_s \in \Theta_s$. Further denote $\theta:=\{\theta_s\}_{s\in\Sc}$, $\Theta:=\prod_{s \in \Sc} \Theta_s$, $\Pi^{\theta}:=\prod_{s\in\Sc}\Pi^{\theta_s}_s$, and $\mathbf{\Pi}:=\{\Pi^{\theta}:\theta\in\Theta\}$ as the class of admissible policies parameterized by the parameter space $\{\theta: \theta\in \Theta\}$.


\paragraph{Initialization.}
Let us also assume that the initial state distribution $\mu_0$ of $N$ agents is sampled from a given distribution $P_0$ over $\Pc^N(\Sc)$, i.e., $\mu_0\sim P_0$; and define the expected total reward function $J(\theta)$ under policy $\Pi^{\theta}$ by
\begin{equation}\label{eqn:total_reward}
 J(\theta)=\E_{\mu_{0}\sim P_0} [\widetilde V^{\Pi^{\theta}}(\mu_{0})].
\end{equation}

\paragraph{Visitation measure.} Denote $\nu_\theta$ as the stationary distribution on $\Xi$ of the Markov process \eqref{eqn:P_N} induced by $\Pi^{\theta}$.
 
Similar to the single-agent RL problem (\citet{agarwal2019theory,fu2020single}), each admissible policy $\Pi^{\theta}$ induces a visitation measure $\sigma_{\theta}(\mu, h)$ on $\Xi$ describing the frequency that policy $\Pi^{\theta}$ visits $(\mu,h)$, with
\begin{equation}\label{eqn:visitation_measure}
    \sigma_{\theta}(\mu, h) :=(1-\gamma) \cdot \sum_{t=0}^{\infty} \gamma^{t} \cdot \mathbb{P}\left(\mu_{t}=\mu, h_{t}=h\mid\Pi^{\theta}\right),
\end{equation}
where $\mu_0\sim P_0$, $h_t\sim\Pi^{\theta}(\cdot\,|\,\mu_t),$ and $\mu_{t+1}\sim \mathbf{P}^N(\cdot\,|\,\mu_t,h_t)$.

\paragraph{Policy gradient theorem.}

In order to find the optimal parameterized policy $\Pi^{\theta}$ which maximizes the expected total reward function $J(\theta)$, the policy optimization step will search for $\theta\in\Theta$ along the gradient direction $\nabla J(\theta)$. Note that computing the gradient $\nabla J(\theta)$ depends on both the action selection, which is directly determined by $\Pi^\theta$, and the visitation measure $\sigma_\theta$ in \eqref{eqn:visitation_measure}, which is indirectly determined by $\Pi^\theta$. 

A simple and elegant  result called 
the policy gradient theorem (Lemma \ref{lemma:policy_grad})   proposed in \citet{sutton1999policy},  reformulates the gradient $\nabla J(\theta)$ in terms of $Q^{\Pi_{\theta}}$ in \eqref{eq:Q_decompostion} and $\nabla \log \Pi^\theta(h\,|\, \mu)$ under the visitation measure $\sigma_\theta$. This result simplifies the gradient computation significantly, and is 
fundamental for actor-critic algorithms.
\begin{Lemma}(\citet{sutton1999policy})\label{lemma:policy_grad}
$\nabla J(\theta)=\frac{1}{1-\gamma}\E_{\sigma_\theta}\left[Q^{\Pi^{\theta}}(\mu, h)\nabla \log\Pi^{\theta}(h\,|\,\mu)\right].$
\end{Lemma}

Now, direct implementation of the actor-critic algorithm with the {\it centralized} policy gradient theorem in Lemma \ref{lemma:policy_grad} suffers from high sample complexity due to the dimension of the Q-function.
Instead, we will show that the exponential decay property of Q-function allows efficient approximation of the policy gradient via {\it localization} and hence a {\it scalable} algorithm to solve (MF-MARL).

\subsection{Neural Policy and Neural Q-function}\label{sec:neurals}



We now turn to  the  localized Q-function $\widehat{Q}^{\Pi}_s(\mu(\Nc_s^k),h(\Nc_s^k))$
(i.e., the approximation of $Q_s^{\Pi}$) and the team-decentralized policy $\Pi_s$, and their parameterization by two-layer neural networks. 
We emphasize that the parameterization framework in this section can be extended to any neural-based single-agent algorithms with convergence guarantee.

\paragraph{Two-Layer Neural Network.} For any input space $\mathcal{X} \subset \R^{d_x}$ with dimension $d_x\in\N$, a two-layer neural network $\widetilde{f}(x;W, b)$ with input $x \in \Xc$ and width $M \in \N$ takes the form of
\begin{equation}\label{eqn:2-layer}
    \widetilde{f}(x; W, b)=\frac{1}{\sqrt{M}} \sum_{m=1}^{M} b_{m} \cdot \operatorname{ReLU}\left(x \cdot [W]_{m}\right).
\end{equation}
Here {{the scaling factor $\frac{1}{\sqrt{M}}$ called the {\it Xavier initialization} (\citet{GB2010})  ensures    the same input variance and the same
gradient variance for all layers}};  the activation function $\mathrm{ReLU}: \mathbb{R} \rightarrow \mathbb{R}$,  defined as $\operatorname{ReLU}(u)=\mathds{1}\{u>0\} \cdot u$; b=$\left\{b_{m}\right\}_{m \in[M]}$ and $W=\left([W]_{1}^{\top}, \ldots,[W]_{M}^{\top}\right)^{\top} \in \mathbb{R}^{M\times d_x}$ in \eqref{eqn:2-layer}
are parameters of the neural network.

Taking advantage of the  homogeneity of ReLU (i.e., $\text{ReLU}(c \cdot u) = c \cdot\text{ReLU} (u)$ for all $c > 0$ and $u \in \R$), we adopt the usual trick  (\citet{cai2019neural}, \citet{wang2019neural},
\citet{ZLS2019}) to fix $b$ throughout the training and only to update  $W$  in the sequel.  Consequently, denote $ \widetilde{f}(x; W, b)$ as $f(x; W)$ when {$b_m=1$} is fixed. $\left[W\right]_{m}$ is initialized according to a multivariate normal distribution $N\left(0, {I_{d_x}} / {d_x}\right)$, where $I_{d_x}$ is the identity matrix of size $d_x$.

\paragraph{Neural Policy.}
For each $s\in\Sc$, denote the tuple $\zeta_s=(\mu(s),h(s))\in\mathbb{R}^{d_{\zeta_s}}$  for {notational} simplicity, where $d_{\zeta_s}:=1+|\Ac|$ is the dimension of  $\zeta_s$. Given the input $\zeta_s = (\mu(s), h(s))$ and parameter $W=\theta_s$ in the two-layer neural network $f(\cdot; \theta_s)$ in \eqref{eqn:2-layer},  the team-decentralized policy $\Pi_s^{\theta_s}$,
called the {\it actor}, is parameterized in the form of an {\it energy-based policy} ,
\begin{equation}\label{eqn:energy_policy}
    \Pi_s^{\theta_s}(h(s) \mid \mu(s))=\frac{\exp [\tau \cdot f((\mu(s), h(s)); \theta_s)]}{\sum_{h'(s)\in\Pc^{N\cdot\mu(s)}(\Ac)} \exp \left[\tau \cdot f\left((\mu(s),h'(s)) ; \theta_s\right)\right]},
\end{equation}
where $\tau$ is the temperature parameter and $f$ is the energy function.

To study the policy gradient for \eqref{eqn:energy_policy}, let us first define a class of feature mappings that is consistent with the representation of two-layer neural networks. This connection between the gradient of a two-layer ReLU neural network and the feature mapping defined in \eqref{eqn:feature} is  crucial  in the convergence analysis  of Theorems \ref{thm:critic_conv} and \ref{thm:actor_conv}.
Specifically,  rewrite  the two-layer neural network in \eqref{eqn:2-layer} as
\begin{equation}\label{eqn:2-layer-feature}
    f(\zeta_s; \theta_s)=\frac{1}{\sqrt{M}} \sum_{m=1}^{M} \operatorname{ReLU}\left(\zeta_s^{\top}[\theta_s]_{m}\right)=\frac{1}{\sqrt{M}} \sum_{m=1}^{M} \mathds{1}\left\{\zeta_s^{\top}[\theta_s]_{m}>0\right\} \cdot \zeta_s^{\top}[\theta_s]_{m}. :=\phi_{\theta_s}(\zeta_s)^{\top} \theta_s.
\end{equation}
Then the feature mapping  $\phi_{\theta_s}=\left(\left[\phi_{\theta_s}\right]_{1}^{\top}, \ldots,\left[\phi_{\theta_s}\right]_{M}^{\top}\right)^{\top}: \mathbb{R}^{d_{\zeta_s}} \rightarrow \mathbb{R}^{M\times d_{\zeta_s}}$ may take the following form:
\vspace{-0.2cm}
\begin{equation}\label{eqn:feature}
    \left[\phi_{\theta_s}\right]_{m}(\zeta_s)=\frac{1}{\sqrt{M}} \cdot \mathds{1}\left\{\zeta_s^{\top}[\theta_s]_{m}>0\right\} \cdot \zeta_s.
\end{equation}
That is, the two-layer neural network $f(\zeta_s; \theta_s)$ may be viewed as the inner product between the feature $\phi_{\theta_s}(\zeta_s)$, and the neural network parameters $\theta_s$. 
Since $f(\zeta_s; \theta_s)$ is almost everywhere differentiable with respect to $\theta_s$, we see  $\nabla_{\theta_s} f(\zeta_s; \theta_s)=\phi_{\theta_s}(\zeta_s)$.

 Furthermore, define a ``centered'' version of the feature $\phi_{\theta_s}$ such that
\begin{equation}\label{eqn:log_pi}
    \Phi(\theta, s, \mu, h):=\phi_{\theta_s}(\mu(s), h(s))-\mathbb{E}_{{h(s)' \sim \Pi_s^{\theta_s}(\cdot \mid \mu(s))}}\left[\phi_{\theta_s}\left(\mu(s), h'(s)\right)\right].
\end{equation}
Note that when policy $\Pi^\theta$ takes the energy-based form \reff{eqn:energy_policy},  $\Phi=\frac{1}{\tau}\nabla_\theta \log\Pi^\theta$. Therefore,

\begin{Lemma}\label{lemma:local_grad_fisher}
For any $\theta \in \Theta$, $s\in\Sc$, $\mu\in\Pc^N(\Sc)$ and $h\in\Hc^N(\mu)$, $\left\|\Phi(\theta, s, \mu, h)\right\|_2 \leq 2$, and
\begin{equation}\label{eqn:grad_nn}
    \nabla_{\theta_s} J\left({\theta}\right)=\frac{\tau}{1-\gamma}\cdot \mathbb{E}_{\sigma_{\theta}}\left[Q^{\Pi^{\theta}}(\mu, h) \cdot \Phi(\theta, s, \mu, h)\right].
\end{equation}
Moreover, for each $s \in \Sc$, define the following localized policy gradient
\begin{equation}\label{eqn:trunc_grad_nn}
    {g}_s(\theta)=\frac{\tau}{1-\gamma}\E_{\sigma_\theta}\left[\Bigg[\sum_{y\in\Nc^k_s}\widehat{Q}_y^{\Pi^{\theta}}(\mu(\Nc^k_y), h(\Nc^k_y)\Bigg]\cdot \Phi(\theta, s, \mu, h)\right],
\end{equation}
with $\widehat{Q}_s^{\Pi^{\theta}}$ in \eqref{eqn:trunc_Q} satisfying the {$(c, \rho)$}-exponential decay property, then there exists a universal constant $c_0>0$ such that
\begin{equation}  \left\|{g}_{s}(\theta)-\nabla_{\theta_{s}} J(\theta)\right\| \leq \frac{c_0 \tau |\Sc|}{1 -\gamma} \rho^{k +1}.
\end{equation}
\end{Lemma}

\paragraph{Neural Q-function.} Note $\widehat{Q}_s^{\Pi^{\theta}}$ in \eqref{eqn:trunc_Q}   is unknown {\it a priori}. To obtain the localized policy gradient \eqref{eqn:trunc_grad_nn}, the neural network \eqref{eqn:2-layer} to parameterize $\widehat{Q}_s^{\Pi^{\theta}}$ is taken as: 
\begin{equation*}
    Q_s(\mu(\Nc^k_s),h(\Nc^k_s); {\omega_s})=f((\mu(\Nc^k_s),h(\Nc^k_s)) ; \omega_s).
\end{equation*}
This $Q_s$ is called the {\it critic}. For simplicity,  denote $\zeta^k_s=(\mu(\Nc^k_s),h(\Nc^k_s))$, with $d_{\zeta^k_s}$ the dimension of $\zeta^k_s$.

\subsection{Actor-Critic}\label{sec:actor-critic}
\paragraph{Critic Update.}
For a fixed policy $\Pi^{\theta}$, it is to estimate  $\widehat{Q}_s^{\Pi^{\theta}}$  of \eqref{eqn:trunc_Q}  by a two-layer neural network $Q_s(\,\cdot\,;\omega_s)$, where $\widehat{Q}^{\Pi^{\theta}}_s$ serves as an approximation to the team-decentralized Q-function $Q^{\Pi^{\theta}}_s$. 

To design the update rule for $\widehat{Q}_s^{\Pi^{\theta}}$, note that the Bellman equation \eqref{eqn:bellman_op} is for ${Q}^{\Pi^{\theta}}_s$ instead of  $\widehat{Q}^{\Pi^{\theta}}_s$. Indeed,  ${Q}^{\Pi^{\theta}}_s$ takes $(\mu,h)$ as the input  while  $\widehat{Q}^{\Pi^{\theta}}_s$  takes the partial information $(\mu(\Nc^k_s),h(\Nc^k_s))$ as the input.  

In order to update parameter $\omega_s$, we substitute $(\mu(\Nc^k_s),h(\Nc^k_s))$ for the state-action pair in the Bellman equation \eqref{eqn:bellman_op}. 
It is therefore necessary to study the error of using $(\mu(\Nc^k_s),h(\Nc^k_s))$ as the input.
Specifically, given a tuple $(\mu_t,h_t, {r_{s}(\mu_t(\Nc_s),h_t(s))}, \mu_{t+1}, h_{t+1})$ sampled from the stationary distribution $\nu_\theta$ of adopting policy $\Pi^{\theta}$, the parameter $\omega_s$ will be updated to minimize the error:
\begin{equation*}\label{eqn:bellman_error}
    (\delta_{s,t})^2=\left[Q_s(\mu_t(\Nc^k_s),h_t(\Nc^k_s);\omega_s)-r_{s}(\mu_t(\Nc_s),h_t(s))-\gamma\cdot Q_s(\mu_{t+1}(\Nc^k_s),h_{t+1}(\Nc^k_s);\omega_s)\right]^2.
\end{equation*}
Estimating $\delta_{s,t}$ depends only on $\mu_t(\Nc_s^k),h_t(\Nc_s^k)$ and can be {\it collected locally}. (See Theorem \ref{thm:critic_conv}).

The neural critic update takes the iterative forms of
\begin{align}\label{eqn:critic_update}
\omega_s(t+1 / 2) & \leftarrow \omega_s(t)-\eta_{\mathrm{critic}} \cdot \delta_{s,t} \cdot \nabla_{\omega_s} Q_s(\mu_t(\Nc^k_s),h_t(\Nc^k_s);\omega_s), \\\label{eqn:critic_update_2}
\quad \omega_s(t+1) & \leftarrow \argmin_{\omega \in \Bc^{\text{critic}}_s}\|\omega-\omega_s(t+1 / 2)\|_{2},\\ \label{eqn:critic_update_3}
\bar{\omega}_s &\leftarrow {(t+1)}/{(t+2)} \cdot \bar{\omega}_s +{1}/{(t+2)} \cdot \omega_s(t+1),
\end{align}
in which  $\eta_{\mathrm{critic}}$ is the learning rate. Here \reff{eqn:critic_update} is the stochastic semigradient step, \eqref{eqn:critic_update_2} is a projection to the parameter space $\Bc^{\text{critic}}_s:=\big\{\omega_s\in\R^{M\times d_{\zeta^k_s}}:\|\omega_s-\omega_s(0)\|_\infty\leq R/\sqrt{M}\big\}$
for some $R>0$, and \reff{eqn:critic_update_3} is the averaging step. {This critic update is summarized in Algorithm \ref{DNTD}.}

\begin{algorithm}[H]
  \caption{\textbf{Localized-Training-Decentralized-Execution Neural Temporal Difference}}
  \label{DNTD}
\begin{algorithmic}[1]
    \STATE \textbf{Input}: Width of the neural network $M$, radius of the constraint set $R$, number of iterations $T_{\mathrm{critic}}$, policy $\Pi^{\theta}$, learning rate $\eta_{\mathrm{critic}}$, localization parameter $k$.
    \STATE \textbf{Initialize}: For all $m \in[M]$ and $s\in\Sc$, sample $b_{m} \sim \operatorname{Unif}(\{-1,1\})$, $[\omega_{s}(0)]_m \sim N\left(0, I_{d_{\zeta^k_s}} / d_{\zeta^k_s}\right)$, $\bar{\omega}_s=\omega_s(0)$.
    \FOR {$t=0$ to $T_{\mathrm{critic}}-2$}
        \STATE Sample $(\mu_t, h_t, \{r_s(\mu_t(\Nc_s), h_t(s))\}_{s\in\Sc}, {\mu_t}', {h_t}')$ from the stationary distribution {$\nu_{\theta}$} of $\Pi^{\theta}$.
        \FOR {$s\in\Sc$}
            \STATE Denote $\zeta^k_{s,t}=(\mu_t(\Nc^k_s),h_t(\Nc^k_s))$, ${\zeta^k_{s,t}}'=({\mu_t}'(\Nc^k_s),{h_t}'(\Nc^k_s))$.
            \STATE Residual calculation:  $\delta_{s,t} \leftarrow Q_s({\zeta^k_{s,t}};\omega_s(t))-r_s(\mu_t(\Nc_s), h_t(s))-\gamma\cdot Q_s({\zeta^k_{s,t}}';\omega_s(t)).$
            \STATE Temporal difference update:
            \STATE $\omega_s(t+1 / 2) \leftarrow \omega_s(t)-\eta_{\mathrm{critic}} \cdot \delta_{s,t} \cdot \nabla_{\omega_s} Q_s({\zeta^k_{s,t}};\omega_s(t))$.
            \STATE Projection onto the parameter space: $\omega_s(t+1) \leftarrow \argmin_{\omega \in \mathcal{B}^{\text{critic}}_s}\|\omega-\omega_s(t+1 / 2)\|_{2}$.
            \STATE Averaging the output: $\bar{\omega}_s \leftarrow \frac{t+1}{t+2} \cdot \bar{\omega}_s +\frac{1}{t+2} \cdot \omega_s(t+1)$.
        \ENDFOR
    \ENDFOR
    \STATE \textbf{Output}: $Q_s(\,\cdot\,;\bar{\omega}_s), \forall s\in\Sc$.
\end{algorithmic}
\end{algorithm}

\paragraph{Actor Update.}

At the iteration step $t$,    a neural network estimation $Q_s(\,\cdot\,;\bar{\omega}_s)$ is given for the localized Q-function $\widehat{Q}_s^{\Pi^{\theta(t)}}$ under the current policy $\Pi^{\theta(t)}$. Let $\left\{\left(\mu_{l}, h_{l}\right)\right\}_{l \in[B]}$ be samples from the state-action visitation measure $\sigma_{\theta(t)}$ of \eqref{eqn:visitation_measure}, and  define an estimator $\widehat{\Phi}(\theta, s, \mu_l, h_l)$ of $\Phi(\theta, s, \mu_l, h_l)$ in \eqref{eqn:log_pi}:
\vspace{-0.2cm}
\begin{equation*}
    \widehat{\Phi}(\theta, s, \mu_l, h_l)=\phi_{\theta_s}(\mu_l(s), h_l(s))-\mathbb{E}_{\Pi_s^{\theta_s}}\left[\phi_{\theta_s}\left(\mu_l(s), h'(s)\right)\right].
\end{equation*}
By Lemma \ref{lemma:local_grad_fisher}, one can  compute the following estimator of $g_s(\theta(t))$ defined in \eqref{eqn:trunc_grad_nn},
\begin{equation}\label{eqn:trunc_grad_nn_est}
    \widehat{g}_s(\theta(t))=\frac{\tau}{(1-\gamma)B}\sum_{l\in[B]}\left[\Bigg[\sum_{y\in\Nc^k_s}{Q}_y\left(\mu_l(\Nc^k_y), h_l(\Nc^k_y);\bar{\omega}_y\right)\Bigg]\cdot \widehat{\Phi}(\theta(t), s, \mu_l, h_l)\right].
\end{equation}
This estimators $\widehat{g}_s$ in \eqref{eqn:trunc_grad_nn_est} only {\it depends locally} on $\left\{\left(\mu_{l}, h_{l}\right)\right\}_{l \in[B]}$. Hence $\widehat{g}$ and $\widehat{\Phi}$ can be computed in a {\it {localized} fashion} after the samples are collected. Similar to the critic update, $\theta_s(t)$ is updated by performing a gradient step with $\widehat{g}_s$, and then projected onto the parameter space $\Bc_s^{\text{actor}}:=\big\{\theta_s\in\R^{M\times d_{\zeta_s}}:\|\theta_s-\theta_s(0)\|_\infty\leq R/\sqrt{M}\big\}$. 

This actor update is summarized in Algorithm \ref{DNPG}.

\paragraph{Sampling from $\nu_\theta$ and the Visitation Measure $\sigma_\theta$.}
In Algorithms \ref{DNTD} and \ref{DNPG}, it is assumed  that one can sample independently from the stationary distribution $\nu_\theta$ and the visitation measure $\sigma_\theta$, respectively.  Such an assumption of sampling from $\nu_\theta$ can be relaxed by either  sampling from a rapidly-mixing Markov chain mixing,
with weakly-dependent sequence of samples (\citet{bhandari2018finite}), or by randomly picking samples from replay buffers consisting of long trajectories, with reduced correlation between samples. 

To sample from the visitation measure $\sigma_\theta$ and computing the unbiased policy gradient estimator,  \citet{konda2000actor} suggests introducing a new MDP such that the next state is sampled from the transition probability with probability $\gamma$, and from the initial distribution with probability $1 - \gamma$. Then the stationary distribution of this new MDP is exactly the visitation measure. Alternatively, \citet{liu2019off} proposes an importance-sampling-based algorithm which enables off-policy evaluation with low variance.

\begin{algorithm}[H]
  \caption{\textbf{Localized-Training-Decentralized-Execution Neural Actor-Critic}}
  \label{DNPG}
\begin{algorithmic}[1]
    \STATE \textbf{Input}: Width of the neural network $M$, radius of the constraint set $R$, number of iterations $T_{\mathrm{actor}}$ and $T_{\mathrm{critic}}$, learning rate $\eta_{\mathrm{actor}}$ and $\eta_{\mathrm{critic}}$, temperature parameter $\tau$, batch size $B$, localization parameter $k$.
    \STATE \textbf{Initialize}: For all $m \in[M]$ and $s\in\Sc$, sample $b_{m} \sim \operatorname{Unif}(\{-1,1\})$, $[\theta_{s}(0)]_m \sim N\left(0, I_{d_{\zeta_s}} / d_{\zeta_s}\right)$.
    \FOR {$t=1$ to $T_{\mathrm{actor}}$}
        \STATE Define the policy 
        \vspace{-0.2cm}\begin{eqnarray*}\Pi^{\theta}(h\mid\mu):=\Prod_{s\in\Sc}\Pi^{\theta_s}_s(h(s) \mid \mu(s))=\Prod_{s\in\Sc}\frac{\exp [\tau \cdot f((\mu(s),h(s)) ; \theta_s)]}{\sum_{h'(s)\in\Hc^N} \exp \left[\tau \cdot f\left((\mu(s),h'(s)) ; \theta_s\right)\right]}.\end{eqnarray*} \vspace{-0.5cm}
        \STATE Output $Q_s(\,\cdot\,;\bar{\omega}_s)$ using Algorithm \ref{DNTD} with the inputs: policy $\Pi^{\theta}$, width of the neural network $M$, radius of the constraint set $R$, number of iterations $T_{\mathrm{critic}}$, learning rate $\eta_{\mathrm{critic}}$ and localization parameter $k$.
        \STATE Sample $\{\mu_l, h_l\}_{l\in[B]}$ from the state-action visitation measure $\sigma_{\theta}$ \eqref{eqn:visitation_measure} of $\Pi^{\theta}$.
        \FOR {$s\in\Sc$}
            \STATE Compute the local gradient estimator $\widehat{g}_s(\theta(t))$ using \eqref{eqn:trunc_grad_nn_est}.
            \STATE Policy update: $\theta_s(t+1/2) \leftarrow \theta_s(t) + \eta_{\mathrm{actor}}\cdot\widehat{g}_s(\theta(t))$
            \STATE Projection onto the parameter space: $\theta_s(t+1) \leftarrow \argmin_{\theta \in \mathcal{B}^{\text{actor}}_s}\|\theta-\theta_s(t+1/2)\|_{2}$.
        \ENDFOR
    \ENDFOR
    \STATE \textbf{Output}: $\{\Pi^{\theta(t)}\}_{t\in[T_{\text{actor}}]}$.
\end{algorithmic}
\end{algorithm}

{\color{red}

}

\section{Convergence of the Critic and Actor Updates}
We now establish the global convergence  for {\DECAC} proposed in Section \ref{sec:algorithm}.

\paragraph{Convergence of the Critic Update.}
The convergence of the decentralized neural critic update in Algorithm \ref{DNTD} relies on the following assumptions.

\begin{Assumption}(Action-Value Function Class)\label{ass:F_infty}For each $s\in\Sc$, $k\in\N$, define
\begin{equation*}
   \mathcal{F}^{s,k}_{R, \infty}=\left\{f(\zeta^k_s)=Q_{s}(\zeta^k_s;\omega_s(0))+\int \mathds{1}\left\{v^{\top} \zeta^k_s>0\right\} \cdot (\zeta^k_s)^{\top} \iota(v)\, d\mu(v):\|\iota(v)\|_{\infty} \leq R\right\},
\end{equation*}
with $\mu: \R^{d_{\zeta^k_s}}\to\R$  the density function of Gaussian distribution $N(0, I_{d_{\zeta^k_s}} / d_{\zeta^k_s})$ and $Q_{s}(\zeta_s^k;\omega_s(0))$ the two-layer neural network under the initial parameter $\omega_s(0)$. We assume that $\widehat{Q}_s^{\Pi^{\theta}} \in \mathcal{F}^{s,k}_{R, \infty}$.
\end{Assumption}

\begin{Assumption}(Regularity of $\nu_\theta$ and $\sigma_\theta$)\label{ass:reg_nu}
    There exists a universal constant $c_{0}>0$ such that for any policy $\Pi^\theta$, any $\alpha \geq 0$, and any $v \in \mathbb{R}^{d_\zeta}$ with $\|v\|_{2}=1$, the stationary distribution $\nu_\theta$ and the state visitation measure $\sigma_\theta$ satisfy
\begin{equation*}
    \mathbb{P}_{\zeta \sim \nu_\theta}\left(\left|v^{\top} \zeta\right| \leq \alpha\right) \leq c_{0} \cdot \alpha,\quad \mathbb{P}_{\zeta \sim \sigma_\theta}\left(\left|v^{\top} \zeta\right| \leq \alpha\right) \leq c_{0} \cdot \alpha.
\end{equation*}
\end{Assumption}

\begin{Remark}
Both Assumption \ref{ass:F_infty} and Assumption \ref{ass:reg_nu} are similar to the standard assumptions in the analysis of single-agent neural actor-critic algorithms (\citet{cai2019neural,liu2019neural,wang2019neural,cayci2021sample}). 

In particular,  Assumption \ref{ass:F_infty} is a regularity condition for $\widehat{Q}_s^{\Pi^{\theta}}$ in \eqref{eqn:trunc_Q}. Here $\mathcal{F}^{s,k}_{R, \infty}$ is a subset of the reproducing kernel Hilbert space (RKHS) induced by the random feature $\mathds{1}\left\{v^{\top} \zeta^k_s>0\right\} \cdot (\zeta^k_s)$ with $v \sim N(0, I_{d_{{\zeta^k_s}}} / d_{{\zeta^k_s}})$ up to the shift of $Q_{s}(\zeta^k_s;\omega_s(0))$ (\citet{rahimi2008uniform}). 
This RKHS is dense in the space of continuous functions on any compact set (\citet{micchelli2006universal,Ji2020Neural}). (See also Section \ref{subsec:critic_notation} for details of the connection between $\mathcal{F}^{s,k}_{R, \infty}$ and the linearizations of two-layer neural networks \eqref{eqn:local_nn}).

Assumption \ref{ass:reg_nu} holds when $\sigma_\theta$ and $\nu_\theta$ have uniformly upper bounded probability densities (\citet{cai2019neural}).
\end{Remark}

\begin{Theorem}(Convergence of Critic Update)\label{thm:critic_conv}
Assume Assumptions \ref{ass:F_infty} and \ref{ass:reg_nu}.
Set $T_{\mathrm{critic}}=\Omega(M)$ and $\eta_{\mathrm{critic}}=\min \{(1-\gamma)/8,(T_{\mathrm{critic}})^{-1/2}\}$ in Algorithm \ref{DNTD}. Then $Q_s(\,\cdot\,;\bar{\omega}_s)$ generated by Algorithm \ref{DNTD} satisfies
\begin{equation}\label{eqn:Q_conv_bound}
    \mathbb{E}_{{\rm{init}}}\left[\left\|Q_s(\,\cdot\,;\bar{\omega}_s) - Q^{\Pi^{\theta}}_s\left(\cdot\right)\right\|^2_{L^2(\nu_\theta)}\right] \leq \Oc\left(\frac{R^3d^{3/2}_{\zeta^k_s}}{M^{1/2}} + \frac{R^{5/2}d_{\zeta^k_s}^{5/4}}{M^{1/4}} + \frac{r_\text{max}^2\gamma^{k+1}}{(1-\gamma)^2}\right),
\end{equation}
where $\|\,f\,\|_{L^2(\nu_\theta)}:=\left(\E_{\zeta\sim\nu_\theta}[f(\zeta)^2]\right)^{1/2}$, and the expectation \reff{eqn:Q_conv_bound} is taken with respect to the random initialization.
\end{Theorem}

Theorem \ref{thm:critic_conv} indicates the trade-off between the {approximation-optimization} error and the localization error.
The first two terms in \eqref{eqn:Q_conv_bound} correspond to the neural network approximation-optimization error, similar to  the single-agent case (\citet{cai2019neural,cayci2021sample}).
This approximation-optimization error decreases when the width of the hidden layer $M$ increases.
Meanwhile, the last term in \eqref{eqn:Q_conv_bound} represents the additional error from using the localized information in \eqref{eqn:critic_update},  unique for the mean-field MARL case.
This localization error and $\gamma^k$  decrease as the number of truncated neighborhood $k$ increases, with more information from a larger neighborhood used in the update.  However, the input dimension $d_{\zeta^k_s}$ and the approximation-optimization error will increase  if the dimension of the  problem increases.

In particular, for a relatively sparse network on $\Sc$,  one can choose $k\ll |\Sc|$ hence $d_{\zeta^k_s}\ll d_{\zeta}$, and Theorem \ref{thm:critic_conv} indicates the superior performance of the localized training scheme in {efficiency} over directly approximating the centralized Q-function.

Proof of Theorem \ref{thm:critic_conv} is presented in Section \ref{app:conv_critic}.



\paragraph{Convergence of the Actor Update.} 

This section establishes the  global convergence of the actor update.
The convergence analysis  consists of two steps. The first step proves the convergence to a stationary point $\widetilde{\theta}$; the second
step controls the gap between the stationary point $\widetilde{\theta}$ and the optimality $\theta^*$ in the overparametrization regime.
The convergence is built under the following assumptions and definition.

\begin{Assumption}(Variance Upper Bound)\label{ass:variance_bound}
    For every $t\in[T_\text{actor}]$ and  $s\in\Sc$, denote $\xi_s(t)=\widehat{g}_s(\theta(t))-\mathbb{E}\left[\widehat{g}_s(\theta(t))\right]$ with $\widehat{g}_s(\theta(t))$ defined in \eqref{eqn:trunc_grad_nn_est}.  Assume there exists $\Sigma>0$ such that $\E\left[\|\xi_s(t)\|_2^2\right]\leq\tau^2\Sigma^2/B$. Here the expectations are taken over  $\sigma_{\theta(t)}$ given $\{\bar{\omega}_s\}_{s\in\Sc}$.
\end{Assumption}

\begin{Assumption}(Regularity Condition on $\sigma_\theta$ and $\nu_\theta$)\label{ass:reg_measure}
There exists an absolute constant $D>0$ such that for every $\Pi^\theta$, the stationary distribution $\nu_\theta$ and the state-action visitation measure $\sigma_\theta$ satisfy
\begin{equation*}
    \left\{\mathbb{E}_{\nu_\theta}\left[\left({\mathrm{d}\sigma_\theta}/{\mathrm{d} \nu_\theta}(\mu,h)\right)^{2}\right]\right\} \leq D^2,
\end{equation*}
where ${\mathrm{d}\sigma_\theta}/{\mathrm{d} \nu_\theta}$ is the Radon-Nikodym derivative of $\sigma_\theta$ with respect to $\nu_\theta$.
\end{Assumption}

\begin{Assumption}(Lipschitz Continuous Policy Gradient)\label{ass:lip_grad}
There exists an absolute constant $L>0$, such that $\nabla_{\theta} J(\theta)$ is $L$-Lipschitz continuous with respect to $\theta$, i.e., for all $\theta_1$, $\theta_2$,
$$\left\|\nabla_{\theta} J(\theta_1) - \nabla_{\theta} J(\theta_2)\right\|_2\leq L\cdot\left\|\theta_1-\theta_2\right\|_2.$$
\end{Assumption}

\begin{Definition}\label{def:station_point}
     $\widetilde{\theta}\in\Bc^{\text{actor}}$ is called a \textit{stationary point} of $J(\theta)$ if for all $\widetilde{\theta}\in\Bc^{\text{actor}}$,
    \begin{equation}
        \nabla_\theta J(\widetilde{\theta})^{\top}(\theta-\widetilde{\theta})\leq 0.
    \end{equation}
\end{Definition}

\begin{Assumption}(Policy Function Class)\label{ass:F_infty_policy}  Define
a function class
\begin{align*}\label{eqn:F_infty_policy}
\mathcal{F}_{R, \infty}=\Bigg\{f(\zeta)=\sum_{s\in\Sc}\Bigg[\phi_{{\theta}_s(0)}(\zeta_s)^\top{\theta}_s(0)+\int \mathds{1}\left\{v^{\top} \zeta_s>0\right\} \cdot (\zeta_s)^{\top} \iota(v)\, d\mu(v)\Bigg]:\|\iota(v)\|_{\infty} \leq R\Bigg\}
\end{align*}
where $\mu: \R^{d_{\zeta_s}}\to\R$ is the density function of the Gaussian distribution $N\left(0, I_{d_{\zeta_s}} / d_{\zeta_s}\right)$ and $\theta(0)$ is the initial parameter. For any stationary point $\widetilde{\theta}$, define the function
\begin{equation*}\label{eqn:u_opt}
    u_{\widetilde{\theta}}(\mu,h) := \frac{d\sigma_{\theta^*}}{d\sigma_{\widetilde\theta}}(\zeta)-\frac{d\bar\sigma_{\theta^*}}{d\bar\sigma_{\widetilde\theta}}(\mu)+\sum_{s\in\Sc}\phi_{\widetilde{\theta}_s}(\zeta_s)^\top\widetilde{\theta}_s,
\end{equation*}
with $\bar\sigma_\theta$ the state visitation measure under policy $\Pi^\theta$, and $\frac{\mathrm{d}\sigma_{\theta^*}}{\mathrm{d}\sigma_{\widetilde\theta}}$,$\frac{\mathrm{d}\bar\sigma_{\theta^*}}{\mathrm{d}\bar\sigma_{\widetilde\theta}}$  the Radon-Nikodym derivatives between corresponding measures. We assume that $u_{\widetilde{\theta}}\in\Fc_{R,\infty}$ for any stationary point $\widetilde{\theta}$.
\end{Assumption}

A few remarks are in place for these Assumption \ref{ass:variance_bound} - Assumption \ref{ass:F_infty_policy}.

\paragraph{Remark.}
All these assumptions are counterparts of  standard assumption in the analysis of single-agent policy gradient method (\citet{pirotta2015policy}, \citet{xu2019sample}, \citet{xu2020improved}, \citet{ZKZB2020}, \citet{wang2019neural}).

In particular, Assumption \ref{ass:variance_bound} and Assumption \ref{ass:reg_measure} hold if the Markov chain \eqref{eqn:P_N} mixes sufficiently fast, and the critic $Q_s(\,\cdot\,;\omega_s)$ has an upper-bounded second moment under $\sigma_{\theta(t)}$ (\citet{wang2019neural}). Note that different from Assumption \ref{ass:reg_nu}, where regularity conditions are imposed separately on $\nu_\theta$ and $\sigma_\theta$, Assumption \ref{ass:reg_measure} imposes the regularity condition directly on the Radon-Nikodym derivative of $\sigma_\theta$ with respect to $\nu_\theta$. This allows the change of measures in the analysis of Theorem \ref{thm:actor_conv}. In general, Assumption \ref{ass:reg_nu} does not necessarily imply Assumption \ref{ass:reg_measure}.

Assumption \ref{ass:lip_grad}  holds when the transition probability and the reward function are both Lipschitz continuous with respect to their inputs (\citet{pirotta2015policy}), or when the reward is uniformly bounded and the score function $\nabla_\theta \Pi^\theta$  is uniformly bounded and Lipschitz continuous with respect to $\theta$ (\citet{ZKZB2020}).

As for Assumption \ref{ass:F_infty_policy}, we first emphasize that $u_{\widetilde{\theta}}(\mu,h)$ is a key element in the proof of Theorem \ref{thm:actor_conv}. More specifically, this assumption is motivated by the well-known Performance Difference Lemma (\citet{KL2002}) in order to characterize the optimality gap of a stationary point $\widetilde{\theta}$.
In particular, it  guarantees that $u_{\widetilde{\theta}}$ can be decomposed into a sum of local functions depending on $\zeta_s$, and that each local function lies in a rich RKHS (see the discussion after Assumption \ref{ass:F_infty}). 

With all these assumptions, we now establish the rate of convergence for Algorithm \ref{DNPG}.

\begin{Theorem}\label{thm:actor_conv}
Assume Assumptions \ref{ass:F_infty} - \ref{ass:F_infty_policy} . Set $
    T_{\mathrm{critic}}=\Omega(M)$, $\eta_{\mathrm{critic}}=\min \{(1-\gamma) / 8,(T_{\mathrm{critic}})^{-1/2}\}$, $\eta_{\mathrm{actor}}=(T_{\mathrm{actor}})^{-1/2}, R=\tau=1$,
    $M=\Omega\left((f(k)|\Ac|)^5(T_{\mathrm{actor}})^8\right)$,
    $\gamma\leq (T_{\mathrm{actor}})^{-2/k}$,
with $f(k):=\max_{s\in\Sc}|\Nc^k_s|$ the size of the largest $k$-neighborhood in the graph $(\Sc, \Ec)$.
Then, the output $\{\theta(t)\}_{t\in[T_{\mathrm{actor}}]}$ of Algorithm \ref{DNPG} satisfies
\begin{equation}
    \min_{t\in[T_{\mathrm{actor}}]}\E\left[J(\theta^*)-J(\theta(t))\right]\leq\Oc\left(|\Sc|^{1/2}B^{-1/2} + |\Sc||\Ac|^{1/4}\left(\gamma^{k/8}+(T_{\mathrm{actor}})^{-1/4}\right)\right).
\end{equation}
\end{Theorem}

The error $\Oc(\gamma^{k/8}|\Sc||\Ac|^{1/4})$ in Theorem \ref{thm:actor_conv},  coming from the localized training, decays exponentially fast as $k$ increases and is negligible  with a careful choice of $k$. According to Theorem \ref{thm:actor_conv}, Algorithm \ref{DNPG} converges  at rate $T_{\mathrm{actor}}^{-1/4}$ with sufficiently large width $M$ and batch size $B$. Technically,  $\{\theta_s(t)\}_{s\in\Sc}$ in Algorithm \ref{DNPG} are updated in parallel and our analysis extends the single agent actor-critic in \citet{cai2019neural} to the multi-agent decentralized case. 

Detailed proof is provided in Section \ref{app:conv_actor}.

\bibliographystyle{informs2014}
\bibliography{refs}

\newpage

\begin{center}
{\Large \bf {\centering Appendix}}
\end{center}

\begin{appendix}

\section{Proof of Lemma \ref{equiv_V_tildeV}} \label{app:section23lemma}


The goal is to show that $V(\mu) = \widetilde V(\mu)$, with the former the value function of \reff{eq:N_agent_reward} subject to the transition probability ${P}$ defined in \reff{eq:N_agent_transition_heter} under a given individual policy $\pi \in \mathfrak{u}$, and the latter the value function of \reff{eq:N_agent_reward_reformu}  subject to the joint transition probability $\mathbf{P}^N$ defined in \reff{eqn:P_N} under the policy $\Pi \in \mathfrak{U}$. The proof consists of two steps. Step 1  shows that $V(\mu)$ can be reformulated as a {\it measured-valued} Markov decision problem. Step 2 shows that the measured-valued Markov decision problem from Step 1 is equivalent to $\widetilde V(\mu)$ in \reff{eq:N_agent_reward_reformu}.

{\it Step 1:} \; Recall that $\mu_{t + 1}: = \frac{1}{N} \sum_{i=1}^N \delta_{s_{t+1}^i}$ with $s_{t + 1}^i$ subject to \reff{eq:N_agent_transition_heter}. First, one can show that $\mu_t$ is a measure-valued Markov decision process under $\pi$. To see this, denote $\Fc_t^s = \sigma(s_t^1, \cdots, s_t^N)$ as the $\sigma$-algebra generated by $s_t^1, \cdots, s_t^N$. Then it suffices to show
\vspace{-0.2cm}
\begin{align} \label{eqn: markov_mu}
    \P(\mu_{t + 1} \;|\; \sigma(\mu_t) \vee \Fc_t^s) = \P(\mu_{t + 1} \;|\; \sigma(\mu_t)), \;\; \P-a.s..
\end{align}
Following similar arguments for Lemma 2.3.1 and  Proposition 2.3.3 in \citet{Dawson1993}, \reff{eqn: markov_mu} holds due to the exchangeability of the individual transition dynamics \reff{eq:N_agent_transition_heter} under $\pi$. \reff{eqn: markov_mu} implies that there exists a joint transition probability induced from \reff{eq:N_agent_transition_heter} under $\pi$, denoted as $\widetilde{\mathbf P}^N$ such that
\vspace{-0.2cm}
\begin{align} \label{eqn: measure_valued_MDP_mu}
\mu_{t + 1} \sim {\widetilde{\mathbf P}}^N(\cdot \; |\; \mu_t, \pi).
\end{align}
Meanwhile, rewrite $V^{\pi}(\mu)$ in \reff{eq:N_agent_reward} by regrouping the agents according to their states
\begin{eqnarray}  
V^{\pi}(\mu) &:=&\mathbb{E}\biggl[ \sum_{t = 0}^\infty\gamma^t\sum_{i=1}^N\frac{1}{N} {{r}(s^i_t,\mu_t(\Nc_{s^i_t}),a^i_t)} \left| \mu_0 = \mu \biggl],\right. \label{mf_objective}\\ 
&=& \mathbb{E}\biggl[ \sum_{t = 0}^\infty\gamma^t\sum_{s\in\Sc}\mu_t(s)\sum_{a\in\Ac}{{r}(s,\mu_t(\Nc_{s}),a)}\pi(s, \mu_t(s))(a) \left| \mu_0 = \mu \biggl].\right. \nonumber
\end{eqnarray}
We see  \reff{eq:N_agent_transition_heter}-\reff{eq:N_agent_reward} is reformulated in an equivalent form of \reff{eqn: measure_valued_MDP_mu}-\reff{mf_objective}.

{\it Step 2:} \; It suffices to show that \reff{eqn: measure_valued_MDP_mu} under $\pi$ is the same as \reff{eqn:P_N} under $\Pi$ and that $V^{\pi}$ in \reff{mf_objective} equals to $\widetilde V^{\Pi}$ in \reff{eq:N_agent_reward_reformu}. To see this, denote $\langle g, \mu \rangle = \sum_{s \in \Sc} g(s) \mu(s)$ for any measurable bounded function $g: \Sc \to \R$, then
\begin{align}
   &\E\big[\langle g, \mu_{t + 1}\rangle \mid \sigma(\mu_t)\big] \nonumber\\
   &=  \frac{1}{N} \E\Big[\sum_{i=1}^N\E\big[g(s^j_{t+1})\mid \sigma(\mu_t) \vee \Fc_t^s \big]\Big] \nonumber\\
   &= \frac{1}{N} \sum_{s' \in \Sc} \sum_{i=1}^N\sum_{a \in \Ac}g(s') {P}(s'\,|\,s^i_t, \mu_{t}(\Nc(s^i_{t})), a) \pi(s_t^i, \mu_t(s_t^i))(a) \nonumber\\
    &= \frac{1}{N} \sum_{s' \in \Sc} g(s')\sum_{s \in \Sc}\sum_{i=1}^N \mathds{1}(s_t^i = s)\sum_{a \in \Ac} {P}(s'\,|\,s^i_t, \mu_{t}(\Nc(s^i_{t})), a) \pi(s_t^i, \mu_t(s_t^i))(a) \nonumber\\
    &= \sum_{s' \in \Sc} g(s')\sum_{s \in \Sc}\mu_t(s)\sum_{a \in \Ac} {P}(s'\,|\,s, \mu_{t}(\Nc(s)), a) \pi(s, \mu_t(s))(a) \nonumber\\
   &=\sum_{s' \in \Sc} g(s') \sum_{s \in \Sc} \mu_t(s)\sum_{h \in \Pc^{N\cdot \mu_t(s)}(\Ac)} \Pi(h\mid\mu_t(s))\sum_{a \in \Ac}{P}(s'\,|\,s, \mu_{t}(\Nc(s)), a) h(s)(a),\label{eqn: measure_valued_MDP}
  \end{align}
where in the last step, the expectation of random variable $h(s)(a)$ with respect to distribution $\Pi(h\; | \; \mu )$ is $\pi(s, \mu_t(s))$. And from the last equality, clearly $\mu_{t + 1}$ evolves according to transition dynamics ${\mathbf P}^N(\cdot| \mu_t, h_t)$ under $\Pi(h_t\mid\mu_t)$. This implies the equivalence of \reff{eqn: measure_valued_MDP_mu} and \reff{eqn:P_N}. As a byproduct, when taking $g(s') = \mathds{1}(s'=s^o)$ for any fixed $s^o \in \Sc$, \reff{eqn: measure_valued_MDP} becomes
\vspace{-0.2cm}
\begin{align*}
    \E\big[\mu_{t+1}(s^o)|\sigma(\mu_t)\big] = \sum_{s \in \Nc(s^o)} \mu_t(s)\sum_{h \in \Pc^{N\cdot \mu_t(s)}(\Ac)} \Pi(h\mid\mu_t(s))\sum_{a \in \Ac}{P}(s^o\,|\,s, \mu_{t}(\Nc(s)), a) h(s)(a),
\end{align*}
where the local structure \reff{eq:N_agent_transition_heter} is used.  This suggests that $\mu_{t+1}(s^o)$ only depends on $\mu_t(\Nc^2_{s^o})$ and $h_t(\Nc_{s^o})$ since $\Nc(s) = \Nc^2(s^o)$ for $s \in \Nc(s^o)$.\\
Now we show that $V^{\pi}(\mu)$ in \reff{mf_objective} and $\widetilde V^{\Pi}(\mu)$ in \reff{eq:N_agent_reward_reformu} are equal.
Take $\widetilde V^{\Pi}$ defined in \reff{eq:N_agent_reward_reformu},
\begin{eqnarray*}\label{eq:N_agent_reward_reformu_app}
\widetilde V^{\Pi}(\mu):&=&  \mathbb{E}_{h_t\sim\Pi(\cdot\,|\,\mu_t),\,\mu_{t+1}\sim \mathbf{P}^N(\cdot\,|\,\mu_t,h_t)}\biggl[ \sum_{t = 0}^\infty \sum_{s \in \Sc}\gamma^t\, r_s(\mu_t(\Nc_{s}), h_t) \left| \mu_0 = \mu \biggl]\right.\\
&=&  \mathbb{E}_{\mu_{t+1}\sim \mathbf{P}^N(\cdot\,|\,\mu_t,h_t)}\biggl[ \sum_{t = 0}^\infty \gamma^t\sum_{s \in \Sc}\, \E_{h_t\sim\Pi(\cdot\,|\,\mu_t)}\big[r_s(\mu_t(\Nc_{s}), h_t)| \mu_t\big] \left| \mu_0 = \mu \biggl]\right.\\
&=& \mathbb{E}_{\mu_{t+1}\sim \mathbf{P}^N(\cdot\,|\,\mu_t,h_t)}\biggl[ \sum_{t = 0}^\infty\gamma^t \sum_{s \in \Sc}\, \sum_{h_t \in \Pc^{N\cdot \mu_t(s)}(\Ac)}r_s(\mu_t(\Nc_{s}), h_t(s)) \Pi(h; \pi) \left| \mu_0 = \mu \biggl]\right.\\
&=& \mathbb{E}_{\mu_{t+1}\sim \mathbf{P}^N(\cdot\,|\,\mu_t,h_t)}\biggl[ \sum_{t = 0}^\infty  \gamma^t\sum_{s \in \Sc}\, \mu_t(s)\sum_{h_t \in \Pc^{N\cdot \mu_t(s)}(\Ac)}\Pi(h_t\mid\mu_t) \sum_{a\in\Ac}{r}(s,\mu_t(\Nc_{s}),a)h(a) \left| \mu_0 = \mu \biggl]\right.\\
&=& \mathbb{E}_{\mu_{t+1}\sim \widetilde{\mathbf{P}}^N(\cdot\,|\,\mu_t,\pi)}\biggl[ \sum_{t = 0}^\infty  \gamma^t\sum_{s \in \Sc}\, \mu_t(s) \sum_{a\in\Ac}{r}(s, \mu_t(\Nc_{s}),a)\pi_t(s, \mu_t(s))(a) \left| \mu_0 = \mu \biggl]\right.\\
&=& V^{\pi}(\mu),
\end{eqnarray*}
where in the last second step, $\mathbf{P}^N$ under $\pi$ is equivalent to ${\widetilde{\mathbf P}}^N$ under $\Pi$, and the expectation of $h_t(s)(a)$ with distribution $\Pi(h_t\mid\mu_t)$ is $\pi(s, \mu_t(s))(a)$ such that 
\begin{align*}
    \sum_{h \in \Pc^{N\cdot \mu_t(s)}(\Ac)}  \Pi(h_t\mid\mu_t) \sum_{a\in\Ac}{r}(s,\mu_t(\Nc_{s}),a)h(a) & = \E_{h \sim \Pi(\cdot\mid\mu_t)} \Big[\sum_{a \in \Ac} {r}(s,\mu_t(\Nc_{s}, a) h(a)\Big]\\
    & = \sum_{a\in\Ac}{r}(s,\mu_t(\Nc_{s}),a)\pi_t(s, \mu_t(s))(a).
\end{align*}
Finally, the decomposition of $\widetilde V(\mu)$ and $Q^{\Pi^\theta}(\mu, h)$ according to the states is straightforward.   \hfill \mbox{\bf Q.E.D.}

\section{Proof of Lemma  \ref{lemma:exp_decay_Q}}\label{app:exponential_decay_lemma}

Let $\mathfrak{P}_{t, s}$ and $\mathfrak{P'}_{t, s}$ be, respectively, distribution of $(\mu_t(\Nc_s), h_t(s))$ and $(\mu'_t(\Nc_s), h'_t(s))$ under policy $\Pi^\theta$. By localized transition kernel \eqref{eq:N_agent_transition_heter}, it is easy to see that for any given $s \in \Sc$, $\mu_{t+1}(s)$ only depends on $\mu_t(\Nc^2_s)$ and $h_t(\Nc_s)$. Then by the local dependency, \reff{eqn:P_N} can be rewritten as
\begin{equation}\label{equ:P_N_loc}
    \mu_{t+1}(s)\sim \mathbf{P}^N_s(\cdot\,|\, \mu_t(\Nc^2_s), h_t(\Nc_s)).
\end{equation}
Due to the local structure of dynamics \reff{equ:P_N_loc} and local dependence of $\Pi^\theta$, the distribution $\mathfrak{P}_{t, s}$, $t \leq \lfloor\frac{ k }{2}\rfloor$ only depends on the initial value $(\mu(\Nc_s^k), h(\Nc_s^k))$. Therefore, $\mathfrak{P}_{t, s} = \mathfrak{P'}_{t, s}$, $t \leq \lfloor\frac{ k }{2}\rfloor$,
\begin{eqnarray*}
& & \bigg|Q_{s}^{\Pi^\theta}\left(\mu(\Nc^{k}_s), \mu(\Nc^{-k}_s), h(\Nc^{k}_s), h(\Nc^{-k}_s)\right) -Q_{s}^{\Pi^\theta}\left(\mu(\Nc^{k}_s), {\mu}'(\Nc^{-k}_s), h(\Nc^{k}_s), h'(\Nc^{-k}_s)\right)\bigg|\\
&=& \sum_{t=\lfloor \frac{k}{2} \rfloor + 1}^\infty  \E_{(\mu_t(\Nc_s), h_t(s)) \sim \mathfrak{P}_{t, s}}\big[ r_s(\mu_t(\Nc_s), h_t(s))\big] - \E_{(\mu'_t(\Nc_s), h'_t(s)) \sim \mathfrak{P}'_{t, s}}\big[ r_s(\mu'_t(\Nc_s), h'_t(s))\big]\\
&\leq & \sum_{t=\lfloor \frac{k}{2} \rfloor + 1}^\infty \gamma^t r_{\text{max}} {\rm TV}(\mathfrak{P}_{t, s},\mathfrak{P'}_{t, s}) \leq \frac{r_{\text{max}}}{1 -\gamma}  \gamma^{\lfloor\frac{ k }{2}\rfloor + 1},
\end{eqnarray*}
where  ${\rm TV}(\mathfrak{P}_{t, s},\mathfrak{P'}_{t, s})$ is total variation between $\mathfrak{P}_{t, s}$ and $\mathfrak{P'}_{t, s}$ that is upper bounded by 1.  \hfill \mbox{\bf Q.E.D.}

\section{Proof of Lemma \ref{lemma:local_grad_fisher}} \label{app:section_local_grad}
For any $\theta \in \Theta$, $s\in\Sc$, $\mu\in\Pc^N(\Sc)$ and $h\in\Hc^N(\mu)$, it is easy to verify that
$\left\|\Phi(\theta, s, \mu, h)\right\|_2\leq \|\zeta_s\|_{2}\leq 2$, by the definitions of the feature mapping $\phi$ in \eqref{eqn:feature} and the center feature mapping $\Phi$ in \eqref{eqn:log_pi}.

To prove \eqref{eqn:grad_nn}, note that by Lemma \ref{lemma:policy_grad} \& the definition of energy-based policy $\Pi_s^{\theta_s}$ \reff{eqn:energy_policy},
\vspace{-0.2cm}
\begin{eqnarray*}
\nabla_{\theta_s}\log \Pi_s^{\theta_s}(h(s)\mid \mu(s)) &=&  \tau \cdot \nabla_{\theta_s}f((\mu(s), h(s)); \theta_s) - \tau \cdot \E_{h(s)' \sim \Pi^{\theta_s}(\cdot \mid \mu(s))}[\nabla_{\theta_s} f(\mu(s), h'(s))]\\
&=& \tau \cdot \phi_{\theta_s}(\mu(s), h(s)) - \tau \cdot \E_{h(s)' \sim \Pi^{\theta_s}(\cdot \mid \mu(s))}[\phi_{\theta_s}(\mu(s), h(s))]\\
&=& \tau\cdot\Phi(\theta, s, \mu, h).
\end{eqnarray*}
The second equality follows from the fact that $\nabla_{\theta_s}f((\mu(s), h(s)); \theta_s)= \phi_{\theta_s}(\mu(s), h(s))$. Therefore, 
\begin{eqnarray*}
\nabla_{\theta_s}J(\theta)=\frac{\tau}{1-\gamma}\mathbb{E}_{\sigma_{\theta}}\left[{Q}^{\Pi^{\theta}}(\mu, h) \cdot\Phi(\theta, s, \mu, h)\right]=\frac{\tau}{1-\gamma}\mathbb{E}_{\sigma_{\theta}}\left[\sum_{y\in\Sc}{Q}_y^{\Pi^{\theta}}(\mu, h) \cdot\Phi(\theta, s, \mu, h)\right],
\end{eqnarray*}
where the second equality is by the decomposition of Q-function in Lemma \ref{equiv_V_tildeV}.

The proof of \reff{eqn:trunc_grad_nn} is based on the exponential decay property in Definition \ref{def:exp_decay}. 
Notice that
\begin{eqnarray}\label{eqn: trunc_grad_nn_equiva}
g_s(\theta)&=&\frac{1}{1-\gamma}\mathbb{E}_{\sigma_{\theta}}\left[\Bigg[\sum_{y\in\Nc^k_s}\widehat{Q}_y^{\Pi^{\theta}}(\mu(\Nc^k_y), h(\Nc^k_y)\Bigg] \nabla_{\theta_s} \log \Pi^{\theta_{s}}(h(s) \mid \mu(s))\right]\nonumber\\
&=&\frac{1}{1-\gamma}\mathbb{E}_{\sigma_{\theta}}\left[\Bigg[\sum_{y\in\Sc}\widehat{Q}_y^{\Pi^{\theta}}(\mu(\Nc^k_y), h(\Nc^k_y)\Bigg] \nabla_{\theta_s} \log \Pi^{\theta_{s}}(h(s) \mid \mu(s))\right].
\end{eqnarray}
This is because for all $y\not\in\Nc^k_s$, $\widehat{Q}_y^{\Pi^{\theta}}(\mu(\Nc^k_y), h(\Nc^k_y)$ is independent of $s$. Consequently, 
$$\mathbb{E}_{\sigma_{\theta}}\left[\left[\sum_{y\not\in\Nc^k_s}\widehat{Q}_y^{\Pi^{\theta}}(\mu(\Nc^k_y), h(\Nc^k_y)\right]\nabla_{\theta_s} \log \Pi^{\theta_{s}}(h(s) \mid \mu(s))\right]=0.$$
Given Lemma \ref{lemma:policy_grad} and \reff{eqn: trunc_grad_nn_equiva}, we have the following bound: 
\begin{eqnarray*}
&&\|{g}_s(\theta)- \nabla_{\theta_s} J(\theta)\|_2 \\
&\leq& \frac{1}{1-\gamma} \sum_{y \in \Sc}\sup_{\substack{\mu\in\Pc^N(\Sc),\\ h\in\Hc^N(\mu)}}\Bigg[\left|\widehat{Q}_{y}^{\Pi^{\theta}}\Big(\mu(\Nc^{k}_y), h(\Nc^{k}_y)\Big) - {Q}_{y}^{\Pi^{\theta}}(\mu, h) \right|  \cdot\|\nabla_{\theta_s} \log \Pi^{\theta_{s}}(h(s) \mid \mu(s))\|_2\Bigg]\\
&\leq & \frac{c_0\tau|\Sc|}{1-\gamma} \rho^{k+1}.
\end{eqnarray*}
The last inequality follows from \reff{eq:exponentialQhat} and $\|\log \Pi^{\theta_{s}}(h(s) \mid \mu(s))\|_2 =\left\|\Phi(\theta, s, \mu, h)\right\|_2\leq 2$ for any $\mu \in \Pc^N(\Sc), h \in \Hc^N(\mu)$.  \hfill \mbox{\bf Q.E.D.}

\section{Proof of Theorems \ref{thm:critic_conv} and \ref{thm:actor_conv}} 
\subsection{Proof of Theorem \ref{thm:critic_conv}: Convergence of Critic Update}\label{app:conv_critic}

This section presents the proof of  convergence of the decentralized neural critic update. It consists of several steps. Section \ref{subsec:critic_notation} introduces necessary notations and definitions.  Section \ref{subsec:critic_1} proves that  the critic update minimizes the projected mean-square Bellman error given a two-layer neural network. 
Section \ref{subsec:critic_2} shows that  the global minimizer of the projected mean-square Bellman error converges to the true team-decentralized Q-function as the width of hidden layer $M\to\infty$.

\subsubsection{Notations}\label{subsec:critic_notation}

Recall that  the set of all state-action (distribution) pairs is denoted as $\Xi:=\cup_{\mu\in\Pc^N(\Sc)}\{\zeta=(\mu,h):h\in\Hc^N(\mu)\}$. For any $\zeta=(\mu,h)\in\Xi$, denote the localized state-action (distribution) pair as $\zeta^k_{s}=(\mu(\Nc^k_s),h(\Nc^k_s))$. Meanwhile, denote $\Xi^k_s=\{\zeta^k_{s}:\zeta\in\Xi\}$ as the set of all possible localized state-action (distribution) pairs. Without loss of generality, assume $\|\zeta^k_{s}\|_2\leq 1$ for any $\zeta^k_{s}\in\Xi^k_s$.

Let $d_\zeta$ denote the dimension of the space $\Xi$. Since $\Pc^N(\Sc)$ has dimension ${(|\Sc|-1)}$ and $\Hc^N(\mu)$ has dimension ${|\Sc|(|\Ac|-1)}$ for any $\mu\in\Pc^N(\Sc)$, the product space $\Xi$ has dimension $d_\zeta={|\Sc| |\Ac|-1}$. Similarly, one can see that the dimension of the space $\Xi^k_s$, denoted by $d_{\zeta^k_s}$, is at most ${f(k)|\Ac|}$, where $f(k):=\max_{s\in\Xc}|\Nc^k_s|$ is the size of the largest $k$-neighborhood in the graph $(\Sc, \Ec)$. 

Let $\R^{\Xi}$ and $\R^{\Xi^k_s}$ be the sets of real-valued square-integrable functions (with respect to $\nu_\theta$) on $\Xi$ and $\Xi^k_s$, respectively. Define the norm $\|\cdot\|_{L^2(\nu_\theta)}$ on $\R^{\Xi}$ by
\begin{equation}\label{eqn:norm}
    \|\,f\,\|_{L^2(\nu_\theta)}:=\left(\E_{\zeta\sim\nu_\theta}[f(\zeta)^2]\right)^{1/2},\quad\forall f\in\R^{\Xi}.
\end{equation}
Note that for any function $f\in\R^{\Xi^k_s}$, a function $\tilde{f}\in\R^{\Xi}$ is called a \textit{natural extension} of $f$ if $\tilde{f}(\zeta)=f(\zeta^k_s)$ for {all $\zeta\in\Xi$}. Since the natural extension is an injective mapping from $\R^{\Xi^k_s}$ to $\R^{\Xi}$, one can view $\R^{\Xi^k_s}$ as a subset of $\R^{\Xi}$. In addition for a function $f\in\R^{\Xi^k_s}$, we use the same notation $f\in\R^{\Xi}$ to denote the natural extension of $f$.

For any closed and convex function class $\Fc\subset\R^{\Xi}$, define the project operator $\text{Proj}_{\Fc}$ from $\R^{\Xi}$ onto $\Fc$ by
\begin{equation}\label{eqn:projection}
    \text{Proj}_{\Fc}(g):= \argmin_{f\in\Fc}\|f-g\|_{L^2(\nu_\theta)}.
\end{equation}
This projection operator $\text{Proj}_{\Fc}$ is non-expansive in the sense that
\begin{equation}\label{eqn:proj_non_expen}
    \|\text{Proj}_{\Fc}(f)-\text{Proj}_{\Fc}(g)\|_{L^2(\nu_\theta)}\leq\|f-g\|_{L^2(\nu_\theta)}.
\end{equation}

Recall that for each state $s\in\Sc$, the critic parameter $\omega_s$ is updated in a localized fashion using information from the $k$-hop neighborhood of $s$. Without loss of generality, let us omit the subscript $s$ of $\omega_s$ in the following presentation, and the result holds for all $s\in\Sc$ simultaneously.

Given an initialization $\omega(0)\in\R^{M\times d_{\zeta^k_s}}$,
define the following function class
\begin{align}\label{eqn:local_nn}
   {\mathcal{F}}_{R,M}=\bigg\{{Q}_{0}(\zeta^k_s;\omega):=\frac{1}{\sqrt{M}}\sum_{m=1}^{M} & \mathds{1}\left\{[\omega(0)]_m^{\top} \zeta^k_s>0\right\} \omega_{m}^{\top} \zeta^k_s:\notag\\
    &\omega \in \mathbb{R}^{M \times d_{\zeta^k_s}}, \|\omega-\omega(0)\|_{\infty} \leq R/\sqrt{M}\bigg\}.
\end{align}
${Q}_{0}(\,\cdot\,;\omega)$ locally linearizes the neural network ${Q}(\,\cdot\,;\omega)$ (with respect to $\omega$) at $\omega(0)$. Any function ${Q}_{0}(\,\cdot\,;\omega)\in\mathcal{F}_{R,M}$ can be viewed as an inner product between the feature mapping $\phi_{\omega(0)}(\cdot)$ defined in \eqref{eqn:feature} and the parameter $\omega$, i.e. ${Q}_{0}(\,\cdot\,;\omega)=\phi_{\omega(0)}(\cdot)^\top \omega$. In addition  it holds that $\nabla_{\omega}{Q}_{0}(\,\cdot\,;\omega)=\phi_{\omega(0)}(\cdot)$. All functions in $\Fc_{R,M}$ share the same feature mapping $\phi_{\omega(0)}(\cdot)$ which only depends on the initialization $\omega(0)$. 


Recall the Bellman operator $\Tc^\theta_s:\R^{\Xi}\to\R^{\Xi}$ defined in \eqref{eqn:bellman_op},
\begin{equation*}
    \mathcal{T}^{\theta}_s {Q}^{\Pi^{\theta}}_s(\mu, h)=\mathbb{E}_{{\mu}'\sim {\mathbf P}^N(\cdot\,|\,\mu,h),\,h'\sim\Pi^{\theta}(\cdot\,|\,\mu)}\left[ r_s(\mu, h)+\gamma\cdot {Q}^{\Pi^{\theta}}_s({\mu}', h')\right], \forall(\mu, h) \in \Xi.
\end{equation*}
The team-decentralized Q-function $Q^{\Pi^\theta}_s$ in  \eqref{eq:Q_decompostion} is the unique fixed point of $\Tc^\theta_s$: $Q^{\Pi^\theta}_s=\Tc^\theta_s Q^{\Pi^\theta}_s$. Now given a general parameterized function class $\Fc$, we aim to learn a $Q_s(\,\cdot\,;\omega)\in\Fc$ to approximate $Q^{\Pi^\theta}_s$ by minimizing the following projected mean-squared Bellman error (PMSBE):
\begin{equation}\label{eqn:PMSBE}
    \min _{\omega} \operatorname{PMSBE}(\omega)=\mathbb{E}_{\zeta \sim \nu_\theta}\left[\Big(Q_s(\,\zeta^k_s\,;\omega)-\text{Proj}_\Fc\mathcal{T}^\theta_s Q_s(\,\zeta^k_s\,;\omega)\Big)^{2}\right].
\end{equation}
In the first step of the convergence analysis, we  take $\Fc=\Fc_{R,M}$ (the locally linearized two-layer neural network defined in \eqref{eqn:local_nn}) and consider the following PMSBE:
\begin{equation}\label{eqn:PMSBE_local_nn}
    \min _{\omega}\mathbb{E}_{\zeta \sim \nu_\theta}\left[\left(Q_0(\,\zeta^k_s\,;\omega)-\text{Proj}_{\Fc_{R,M}}\mathcal{T}^\theta_s Q_0(\,\zeta^k_s\,;\omega)\right)^{2}\right].
\end{equation}
We will show in Section \ref{subsec:critic_1} that the output of Algorithm \ref{DNTD} converges to the global minimizer of \eqref{eqn:PMSBE_local_nn}.

\subsubsection{Convergence to the Global Minimizer in \texorpdfstring{$\Fc_{R,M}$}{Lg}}\label{subsec:critic_1}
The following lemma guarantees the existence and the uniqueness of the global minimizer of MSPBE that corresponds to the projection onto $\Fc_{R,M}$ in \eqref{eqn:PMSBE_local_nn}.
\begin{Lemma}(Existence and Uniqueness of the Global Minimizer in $\Fc_{R,M}$)\label{lemma:global_opt}
For any $b \in \mathbb{R}^{M}$ and $\omega(0) \in \mathbb{R}^{M\times d_{\zeta^k_s}}$, there exists an  $\omega^{*}$ such that ${Q}_{0}\left(\,\cdot\,; \omega^{*}\right)\in\Fc_{R,M}$ is unique almost everywhere in $\Fc_{R,M}$ and is the global minimizer of MSPBE that corresponds to the projection onto $\Fc_{R,M}$ in \eqref{eqn:PMSBE_local_nn}.
\end{Lemma}
\begin{proof}{Proof of Lemma \ref{lemma:global_opt}}
We first show that the operator $\mathcal{T}^\theta_s:\R^{\Xi}\to\R^{\Xi}$ \eqref{eqn:bellman_op} is a $\gamma$-contraction in the $L^2(\nu_\theta)$-norm.
$$
\begin{aligned}
&\|\mathcal{T}^{\theta}_s Q_{1}-\mathcal{T}^{\theta}_s Q_{2}\|^2_{L^2(\nu_\theta)} = \mathbb{E}_{\zeta \sim \nu_\theta}\left[\left(\mathcal{T}^{\theta}_s Q_{1}(\zeta)-\mathcal{T}^{\theta}_s Q_{2}(\zeta)\right)^{2}\right] \\
=&\,\gamma^{2} \mathbb{E}_{\zeta \sim \nu_\theta}\left[\left(\mathbb{E}\left[Q_{1}\left(\zeta'\right)-Q_{2}\left(\zeta'\right) \big| \zeta'=({\mu}', h'),{\mu}'\sim P^N(\cdot\,|\,\zeta), h'\sim\Pi^\theta(\cdot\,|\,{\mu}')\right]\right)^{2}\right] \\
\leq&\, \gamma^{2}\mathbb{E}_{\zeta \sim \nu_\theta}\left[\mathbb{E}\left[\left(Q_{1}\left(\zeta'\right)-Q_{2}\left(\zeta'\right)\right)^2 \big| \zeta'=({\mu}', h'),{\mu}'\sim P^N(\cdot\,|\,\zeta), h'\sim\Pi^\theta(\cdot\,|\,{\mu}')\right]\right]\\
=&\, \gamma^{2}\mathbb{E}_{\zeta' \sim \nu_\theta}[\left(Q_{1}\left(\zeta'\right)-Q_{2}\left(\zeta'\right)\right)^2]=\, \gamma^{2}\|Q_{1} - Q_{2}\|^2_{L^2(\nu_\theta)},
\end{aligned}
$$
where the first inequality follows from H\"older's inequality for the conditional expectation and the third equality stems from the fact that  $\zeta^{\prime}$ and
$\zeta$ have the same stationary distribution $\nu_\theta$.

Meanwhile, the projection operator $\text{Proj}_{\Fc_{R,M}}:\R^{\Xi}\to\Fc_{R,M}$ is non-expansive. Therefore, the operator $\text{Proj}_{\Fc_{R,M}}\mathcal{T}^\theta_s:\Fc_{R,M}\to\Fc_{R,M}$ is $\gamma$-contraction in the $L^2(\nu_\theta)$-norm. Hence $\text{Proj}_{\Fc_{R,M}}$ admits a unique fixed point ${Q}_{0}\left(\,\cdot\,; \omega^{*}\right)\in\Fc_{R,M}$. By definition, ${Q}_{0}\left(\,\cdot\,; \omega^{*}\right)$ is the global minimizer of MSPBE that corresponds to the projection onto $\Fc_{R,M}$ in \eqref{eqn:PMSBE_local_nn}.
\end{proof}
\hfill \mbox{\bf Q.E.D.}

We will show  that  the function class $\Fc_{R,M}$ will approximately become $\Fc^{s,k}_{R,\infty}$ (defined in Assumption \ref{ass:F_infty}) as $M\to\infty$, where $\Fc^{s,k}_{R,\infty}$ is a rich reproducing kernel Hilbert space (RKHS). Consequently,  ${Q}_{0}\left(\,\cdot\,; \omega^{*}\right)$ will become the global minimum of the MSPBE \eqref{eqn:PMSBE_local_nn} on $\Fc^{s,k}_{R,\infty}$ given Lemma \ref{lemma:global_opt}. 
Moreover, by using similar argument and technique developed in  \cite[Theorem 4.6]{cai2019neural}, we can establish the convergence of Algorithm \ref{DNTD} to ${Q}_{0}\left(\,\cdot\,; \omega^{*}\right)$ as the following.
\begin{Theorem}(Convergence to ${Q}_{0}\left(\,\cdot\,; \omega^{*}\right)$)\label{thm:critic_conv_local}
Set $\eta_{\mathrm{critic}}=\min \{(1-\gamma) / 8,1 / \sqrt{T_{\text{critic}}}\}$ in Algorithm \ref{DNTD}. 
Then the output $Q_s(\,\cdot\,;\bar{\omega})$ of Algorithm \ref{DNTD} satisfies
$$
\mathbb{E}_{{\rm{init}}}\left[\left\|Q_s(\,\cdot\,;\bar{\omega}) - Q_{0}\left(\,\cdot\,; \omega^{*}\right)\right\|^2_{L^2(\nu_\theta)}\right] \leq \Oc\left(\frac{R^3d_{\zeta^k_s}^{3/2}}{\sqrt{M}} + \frac{R^{5/2}d_{\zeta^k_s}^{5/4}}{\sqrt[4]{M}} + \frac{R^2d_{\zeta^k_s}}{\sqrt{T_{\mathrm{critic}}}}\right),
$$
where the expectation is taken with respect to the random initialization.
\end{Theorem}
The proof of Theorem \ref{thm:critic_conv_local} is  straightforward  from  \cite[Theorem 4.6]{cai2019neural} and hence  omitted.

\subsubsection{Convergence to \texorpdfstring{$Q^{\Pi^\theta}_s$}{Lg}}\label{subsec:critic_2}
Next, we analyze the error between the global minimizer of \eqref{eqn:PMSBE_local_nn} and the team-decentralized Q-function $Q^{\Pi^\theta}_s$ (defined in \eqref{eq:Q_decompostion}) to complete the convergence analysis. Different from the single-agent case as in \citet{cai2019neural}, we have to bound an additional error from using the localized information in the critic update, in addition to the neural network approximation-optimization error.

\begin{proof}{Proof of Theorem \ref{thm:critic_conv}}
First recall that by Lemma \ref{lemma:exp_decay_Q}, $Q^{\Pi^\theta}_s$ satisfies the $(c,\rho)$-exponential decay property in Definition \ref{def:exp_decay}, with $c=\frac{r_\text{max}}{1-\gamma}$, $\rho=\sqrt{\gamma}$. Now, let $\widehat{Q}^{\Pi^\theta}_s$ be any localized Q-function in \eqref{eqn:trunc_Q}, then
\begin{equation}\label{eqn:diff_Q_trunc}
    \left|Q^{\Pi^\theta}_s(\zeta) - \widehat{Q}^{\Pi^\theta}_s(\zeta^k_s)\right|\leq c\rho^{k+1}, \quad\forall \zeta\in\Xi.
\end{equation}

By the triangle inequality and $(a+b)^2\leq 2(a^2+b^2)$,
\begin{align}\label{eqn:critic_derive_1}
    \left\|Q_s(\,\cdot\,;\bar{\omega}) - Q^{\Pi^\theta}_s(\cdot)\right\|^2_{L^2(\nu_\theta)} \leq& \left(\left\|Q_s(\,\cdot\,;\bar{\omega}) - Q_{0}\left(\,\cdot\,; \omega^{*}\right)\right\|_{L^2(\nu_\theta)} + \left\|Q^{\Pi^\theta}_s(\cdot) - Q_{0}\left(\,\cdot\,; \omega^{*}\right)\right\|_{L^2(\nu_\theta)}\right)^2\notag\\
    \leq &\, 2\left(\left\|Q_s(\,\cdot\,;\bar{\omega}) - Q_{0}\left(\,\cdot\,; \omega^{*}\right)\right\|^2_{L^2(\nu_\theta)} + \left\|Q^{\Pi^\theta}_s(\cdot) - Q_{0}\left(\,\cdot\,; \omega^{*}\right)\right\|^2_{L^2(\nu_\theta)}\right).
\end{align}
The first term in \eqref{eqn:critic_derive_1} is studied in Theorem \ref{thm:critic_conv_local} and it suffices to bound the second term. By interpolating two intermediate terms $\widehat{Q}^{\Pi^\theta}_s$ and $\text{Proj}_{\Fc_{R,M}}\widehat{Q}^{\Pi^\theta}_s$, we have
\begin{align}\label{eqn:critic_derive_2}
    \left\|Q^{\Pi^\theta}_s(\cdot) - Q_{0}\left(\,\cdot\,; \omega^{*}\right)\right\|_{L^2(\nu_\theta)} \leq& \underbrace{\left\|Q^{\Pi^\theta}_s(\cdot) - \widehat{Q}^{\Pi^\theta}_s(\cdot)\right\|_{L^2(\nu_\theta)} }_{(\text{I})}+ \underbrace{\left\|\widehat{Q}^{\Pi^\theta}_s(\cdot) - \text{Proj}_{\Fc_{R,M}}\widehat{Q}^{\Pi^\theta}_s(\cdot)\right\|_{L^2(\nu_\theta)}}_{(\text{II})}\notag \\
    &+ \underbrace{\left\|Q_{0}\left(\,\cdot\,; \omega^{*}\right) - \text{Proj}_{\Fc_{R,M}}\widehat{Q}^{\Pi^\theta}_s(\cdot)\right\|_{L^2(\nu_\theta)}}_{(\text{III})}.
\end{align}
First, we have $(\text{I})\leq c\rho^{k+1}$  according to \eqref{eqn:diff_Q_trunc}.
To bound (III), we have 
\begin{align}\label{eqn:critic_derive_3}
    (\text{III})
    =& \left\|\text{Proj}_{\Fc_{R,M}}\Tc^\theta_s Q_{0}\left(\,\cdot\,; \omega^{*}\right) - \text{Proj}_{\Fc_{R,M}}\widehat{Q}^{\Pi^\theta}_s(\cdot)\right\|_{L^2(\nu_\theta)}\notag\\
    \leq& \left\|\text{Proj}_{\Fc_{R,M}}\Tc^\theta_s Q_{0}\left(\,\cdot\,; \omega^{*}\right) - \text{Proj}_{\Fc_{R,M}}\Tc^\theta_s{Q}^{\Pi^\theta}_s(\cdot)\right\|_{L^2(\nu_\theta)} + \left\|\text{Proj}_{\Fc_{R,M}}\Tc^\theta_s {Q}^{\Pi^\theta}_s(\cdot) - \text{Proj}_{\Fc_{R,M}}\widehat{Q}^{\Pi^\theta}_s(\cdot)\right\|_{L^2(\nu_\theta)}\notag\\
    \leq&\,\gamma\left\|Q_{0}\left(\,\cdot\,; \omega^{*}\right) - {Q}^{\Pi^\theta}_s(\cdot)\right\|_{L^2(\nu_\theta)} + \left\|\Tc^\theta_s {Q}^{\Pi^\theta}_s(\cdot) - \widehat{Q}^{\Pi^\theta}_s(\cdot)\right\|_{L^2(\nu_\theta)}\notag\\
    =&\,\gamma\left\|Q_{0}\left(\,\cdot\,; \omega^{*}\right) - {Q}^{\Pi^\theta}_s(\cdot)\right\|_{L^2(\nu_\theta)} + \underbrace{\left\|Q^{\Pi^\theta}_s(\cdot) - \widehat{Q}^{\Pi^\theta}_s(\cdot)\right\|_{L^2(\nu_\theta)} }_{(\text{I})}\notag\\
    \leq&\,\gamma\left\|Q_{0}\left(\,\cdot\,; \omega^{*}\right) - {Q}^{\Pi^\theta}_s(\cdot)\right\|_{L^2(\nu_\theta)} + c\rho^{k+1}.
\end{align}
The first line in \eqref{eqn:critic_derive_3} is due to the fact that $Q_0(\cdot;\omega^*)$ is the unique fixed point of the operator $\text{Proj}_{\Fc_{R,M}}\Tc^\theta_s$, (as proved in Lemma \ref{lemma:global_opt}); the third line in \eqref{eqn:critic_derive_3} is because the operator $\text{Proj}_{\Fc_{R,M}}\Tc^\theta_s$ is a $\gamma$-contraction in the $L^2(\nu_\theta)$ norm, and $\text{Proj}_{\Fc_{R,M}}$ is non-expansive; the fourth line in \eqref{eqn:critic_derive_3} uses the fact that $Q^{\Pi^\theta}_s$ is the unique fixed point of $\Tc^\theta_s$; and the last line comes from the fact that $(\text{I})\leq c\rho^{k+1}$. Therefore, combining the self-bounding inequality \eqref{eqn:critic_derive_3} with \eqref{eqn:critic_derive_2} and the bound on $(\text{I})$ gives us
\begin{align*}
    \left\|Q^{\Pi^\theta}_s(\cdot) - Q_{0}\left(\,\cdot\,; \omega^{*}\right)\right\|_{L^2(\nu_\theta)} \leq& \,\frac{1}{1-\gamma}\left( 2c\rho^{k+1} + \underbrace{\left\|\widehat{Q}^{\Pi^\theta}_s(\cdot) - \text{Proj}_{\Fc_{R,M}}\widehat{Q}^{\Pi^\theta}_s(\cdot)\right\|_{L^2(\nu_\theta)}}_{(\text{II})}\right),
\end{align*}
and consequently,
\begin{align}\label{eqn:critic_derive_4}
    \left\|Q^{\Pi^\theta}_s(\cdot) - Q_{0}\left(\,\cdot\,; \omega^{*}\right)\right\|_{L^2(\nu_\theta)}^2 \leq& \,\frac{1}{(1-\gamma)^2}\left( 8c^2\rho^{2k+2} + 2\underbrace{\left\|\widehat{Q}^{\Pi^\theta}_s(\cdot) - \text{Proj}_{\Fc_{R,M}}\widehat{Q}^{\Pi^\theta}_s(\cdot)\right\|_{L^2(\nu_\theta)}^2}_{(\text{II})}\right).
\end{align}
Plugging \eqref{eqn:critic_derive_4} into \eqref{eqn:critic_derive_1} yields
\begin{align}\label{eqn:critic_derive_5}
    &\E_{\rm{init}}\left[\left\|Q_s(\,\cdot\,;\bar{\omega}) - Q^{\Pi^\theta}_s(\cdot)\right\|^2_{L^2(\nu_\theta)}\right]\notag\\ \leq&\,2\left(\E_{\rm{init}}\left[\left\|Q_s(\,\cdot\,;\bar{\omega}) - Q_{0}\left(\,\cdot\,; \omega^{*}\right)\right\|^2_{L^2(\nu_\theta)}\right] + \E_{\rm{init}}\left[\left\|Q^{\Pi^\theta}_s(\cdot) - Q_{0}\left(\,\cdot\,; \omega^{*}\right)\right\|^2_{L^2(\nu_\theta)}\right]\right)\notag\\
    \leq&\,\Oc\left(\frac{R^3d^{3/2}_{\zeta^k_s}}{\sqrt{M}} + \frac{R^{5/2}d_{\zeta^k_s}^{5/4}}{\sqrt[4]{M}} + \frac{R^2d_{\zeta^k_s}}{\sqrt{T}} + c^2\rho^{2k+2}\right)+\frac{4}{(1-\gamma)^2}\E_{\rm{init}}\left[\underbrace{\left\|\widehat{Q}^{\Pi^\theta}_s(\cdot) - \text{Proj}_{\Fc_{R,M}}\widehat{Q}^{\Pi^\theta}_s(\cdot)\right\|_{L^2(\nu_\theta)}^2}_{(\text{II})}\right].
\end{align}

Term (II)  measures the distance between $\widehat{Q}^{\Pi^\theta}_s$ and the class $\Fc_{R,M}$. As discussed in Section \ref{subsec:critic_notation}, the function class $\Fc_{R,M}$ converges to $\Fc^{s,k}_{R,\infty}$ (defined in Assumption \ref{ass:F_infty}) as $M\to\infty$. Consequently, term (II) decreases as the neural network gets wider. To quantitatively characterize the approximation error between $\Fc_{R,M}$ and $\Fc^{s,k}_{R,\infty}$,  one needs the following lemma from \citet{rahimi2008uniform} and \cite[Proposition 4.3]{cai2019neural}:

\begin{Lemma}\label{lemma:F_infty}
    Assume Assumption \ref{ass:F_infty}, we have
    \begin{equation}\label{eqn:F_infty}
        \E_{\rm{init}}\left[\underbrace{\left\|\widehat{Q}^{\Pi^\theta}_s(\cdot) - \text{Proj}_{\Fc_{R,M}}\widehat{Q}^{\Pi^\theta}_s(\cdot)\right\|_{L^2(\nu_\theta)}^2}_{(\text{II})}\right]\leq\Oc\left(\frac{R^2d_{\zeta^k_s}}{M}\right).
    \end{equation}
\end{Lemma}

With this lemma, Theorem \ref{thm:critic_conv} follows immediately by plugging \eqref{eqn:F_infty} into \eqref{eqn:critic_derive_5}, and setting $c=\frac{r_\text{max}}{1-\gamma}$, $\rho=\sqrt{\gamma}$, $T_\text{critic}=\Omega(M)$ in \eqref{eqn:critic_derive_5}.
\end{proof} \hfill \mbox{\bf Q.E.D.}

\subsection{Proof of Theorem \ref{thm:actor_conv}: Convergence of Actor Update}\label{app:conv_actor}


The proof of Theorem \ref{thm:actor_conv} consists of two steps: the first step in Section \ref{subsec:actor_1} shows that the actor update converges to a stationary point of $J$ \eqref{eqn:total_reward}, and the second step in Section \ref{subsec:actor_2} bridges the gap between the stationary point and the optimality.

For the rest of this section, we use $\eta$ to denote $\eta_{\text{actor}}$ and $\Bc_s$ to denote $\Bc_s^{\text{actor}}:=\big\{\theta_s\in\R^{M\times d_{\zeta_s}}:\|\theta_s-\theta_s(0)\|_\infty\leq R/\sqrt{M}\big\}$ for ease of notation. Meanwhile, define $\Bc=\prod_{s\in\Sc}\Bc_s$, the product space of $\Bc_s$'s, which is a convex set in $\R^{M\times d_{\zeta}}$.

\subsubsection{Convergence to Stationary Point}\label{subsec:actor_1}

\begin{Definition}\label{def:stationary}
    A point $\widetilde{\theta}\in\Bc$ is called a stationary point of $J(\cdot)$ if it holds that
\begin{equation}\label{eqn:stationary_theta}
    \nabla_\theta J(\widetilde{\theta})^{\top}(\theta-\widetilde{\theta})\leq 0,\quad\forall\theta\in\Bc.
\end{equation}
\end{Definition}

Define the following mapping $G$ from $\R^{M\times d_{\zeta}}$ to itself:
\begin{equation}\label{eqn:grad_map}
    G(\theta):=\eta^{-1}\cdot\left[\text{Proj}_\Bc\left(\theta+\eta\cdot\nabla_\theta J(\theta)\right)-\theta\right].
\end{equation}
It is  well-known that \eqref{eqn:stationary_theta} holds if and only if $G(\widetilde{\theta})=0$ (\citet{sra2012optimization}). Now denote $\rho(t):=G(\theta(t))$, where $\theta(t)=\{\theta_s(t)\}_{s\in\Sc}$ is the actor parameter updated in Algorithm \ref{DNPG} in iteration $t$. 

To show that Algorithm \ref{DNPG} converges to a stationary point, we focus on analyzing $\|\rho(t)\|_2$.

\begin{Theorem}\label{thm:actor_stationary}
Assume Assumptions \ref{ass:variance_bound} - \ref{ass:lip_grad}. Set $\eta = (T_{\mathrm{actor}})^{-1/2}$ and assume $1-L\eta\geq1/2$, where $L$ is the Lipschitz constant in Assumotion \ref{ass:lip_grad}.
Then the output $\{\theta(t)\}_{t\in[T_{\mathrm{actor}}]}$ of Algorithm \ref{DNPG} satisfies
\begin{align}\label{eqn:actor_stat}
    \min_{t\in[T_{\text{actor}}]}\E\left[\|\rho(t)\|^2_2\right]\leq\frac{8\tau^2\Sigma^2|\Sc|}{B}+\frac{4}{\sqrt{T_{\text{actor}}}}\E[J(\theta({T_{\text{actor}}}+1))-J(\theta(1))]+\epsilon_Q(T_{\text{actor}}).
\end{align}
Here $\epsilon_Q$ measures the error accumulated from the critic steps which is defined as
\begin{align}\label{eqn:actor_stat_Q}
    \epsilon_Q(T_{\text{actor}})=&\,\frac{32\tau DRd_{\zeta_s}^{1/2}|\Sc|}{(1-\gamma)\eta T_{\text{actor}}}\cdot\sum_{t=1}^{T_{\text{actor}}} \sum_{s\in\Sc}\E\left[\left\|Q_s(\,\cdot\,;\bar{\omega}_s,t) - Q_s^{\Pi^{\theta(t)}}\left(\cdot\right)\right\|_{L^2(\nu_{\theta(t)})}\right]\notag\\
    &+\frac{16\tau^2D^2|\Sc|^2}{(1-\gamma)^2 T_{\text{actor}}}\cdot\sum_{t=1}^{T_{\text{actor}}}\sum_{s\in\Sc}\E\left[\left\|Q_s(\,\cdot\,;\bar{\omega}_s,t) - Q_s^{\Pi^{\theta(t)}}\left(\cdot\right)\right\|^2_{L^2(\nu_{\theta(t)})}\right],
\end{align}
where $\{Q_{s}(\,\cdot\,;\bar{\omega}_s,t)\}_{s\in\Sc}$ is the output of the critic update at step $t$ in Algorithm \ref{DNPG}. All expectations in \eqref{eqn:actor_stat} and \eqref{eqn:actor_stat_Q} are taken over all randomness in Algorithm \ref{DNTD} and Algorithm \ref{DNPG}.
\end{Theorem}

\begin{proof}{Proof of Theorem \ref{thm:actor_stationary}}

Let $t\in[T_{\mathrm{actor}}]$, we first lower bound the difference between the expected total rewards of $\Pi^{\theta(t+1)}$ and $\Pi^{\theta(t)}$.
By Assumption \ref{ass:lip_grad}, $\nabla_{\theta} J\left(\theta\right)$ is $L$-Lipschitz continuous. Hence by Taylor's expansion,
\begin{equation}\label{eqn:actor_ana_1}
J\left(\theta(t+1)\right)-J\left(\theta(t)\right) \geq \eta \cdot \nabla_{\theta} J\left(\theta(t)\right)^{\top} \delta(t)-L / 2 \cdot\left\|\theta(t+1)-\theta(t)\right\|_{2}^{2},
\end{equation}
where $\delta(t)=\left(\theta(t+1)-\theta(t)\right) / \eta$. Meanwhile denote $\xi_s(t)=\widehat{g}_s(\theta(t))-\mathbb{E}\left[\widehat{g}_s(\theta(t))\right]$, where $\widehat{g}_s(\theta(t))$ is defined in \eqref{eqn:trunc_grad_nn_est} and the expectation is taken over $\sigma_{\theta(t)}$ given $\{\bar{\omega}_s\}_{s\in\Sc}$. Then
\begin{align}\label{eqn:actor_ana_2}
    \nabla_{\theta} J\left(\theta(t)\right)^{\top} \delta(t)&=\sum_{s\in\Sc}\nabla_{\theta_s} J\left(\theta(t)\right)^{\top} \delta_s(t)\notag\\
    &=\sum_{s\in\Sc}\left[\left(\nabla_{\theta_s} J\left(\theta(t)\right)-\mathbb{E}\left[\widehat{g}_s(\theta(t))\right]\right)^{\top} \delta_{s}(t)-\xi_s(t)^{\top} \delta_{s}(t)+\widehat{g}_s(\theta(t))^{\top} \delta_{s}(t)\right],
\end{align}
where $\delta_s(t):=\left(\theta_s(t+1)-\theta_s(t)\right) / \eta$. The first term in \eqref{eqn:actor_ana_2} represents the error of estimating $\nabla_{\theta_s} J\left(\theta(t)\right)$ using $$\mathbb{E}\left[\widehat{g}_s(\theta(t))\right]=\frac{1}{1-\gamma}\mathbb{E}_{\sigma_{\theta(t)}}\left[\Bigg[\sum_{y\in\Nc^k_s}{Q}_y\left(\mu(\Nc^k_y), h(\Nc^k_y);\bar{\omega}_y,t\right)\Bigg]\nabla_{\theta_s} \log \Pi^{\theta_{s}}(h(s) \mid \mu(s))\right].$$

To bound the first term, first notice that
$$\mathbb{E}\left[\widehat{g}_s(\theta(t))\right]=\frac{1}{1-\gamma}\mathbb{E}_{\sigma_{\theta(t)}}\left[\left[\sum_{y\in\Sc}{Q}_y\left(\mu(\Nc^k_y), h(\Nc^k_y);\bar{\omega}_y,t\right)\right]\nabla_{\theta_s} \log \Pi^{\theta_{s}}(h(s) \mid \mu(s))\right].$$
This is because for all $y\not\in\Nc^k_s$, ${Q}_y\left(\mu(\Nc^k_y), h(\Nc^k_y);\bar{\omega}_y\right)$ is independent of $s$ and consequently, we can verify that
$$\mathbb{E}_{\sigma_{\theta(t)}}\left[\left[\sum_{y\not\in\Nc^k_s}{Q}_y\left(\mu(\Nc^k(y)), h(\Nc^k(y));\bar{\omega}_y,t\right)\right]\nabla_{\theta_s} \log \Pi^{\theta_{s}}(h(s) \mid \mu(s))\right]=0.$$
Therefore, following the similar computation in Lemma D.2, \citet{cai2019neural}, we have
\begin{align}\label{eqn:actor_ana_3}
    \left|\left(\nabla_{\theta_s} J\left(\theta(t)\right)-\mathbb{E}\left[\widehat{g}_s(\theta(t))\right]\right)^{\top} \delta_{s}(t)\right|\leq\frac{4\tau DRd_{\zeta_s}^{1/2}}{(1-\gamma)\eta}\sum_{s\in\Sc}\left\|Q_s(\,\cdot\,;\bar{\omega}_s,t) - Q_s^{\theta(t)}\left(\cdot\right)\right\|_{L^2(\nu_{\theta(t)})}.
\end{align}

To bound the second term in \eqref{eqn:actor_ana_2}, we simply have
\begin{equation}\label{eqn:actor_ana_4}
    \xi_s(t)^{\top} \delta_{s}(t)\le \|\xi_s(t)\|_2^2 + \|\delta_{s}(t)\|_2^2.
\end{equation}

To handle the last term in \eqref{eqn:actor_ana_2}, we have
\begin{align}\label{eqn:actor_ana_5}
&\widehat{g}_s(\theta(t))^{\top} \delta_{s}(t)-\|\delta_{s}(t)\|^2_2=\eta^{-1}\cdot(\eta\widehat{g}_s(\theta(t))-\left(\theta_s(t+1)-\theta_s(t)\right))^{\top} \delta_{s}\notag\\
=&\eta^{-1}\cdot\left(\theta_s(t+1/2)-\text{Proj}_{\Bc_s}(\theta_s(t+1/2))\right)^{\top} \delta_{s}(t)\notag\\
=&\eta^{-2}\cdot\left(\theta_s(t+1/2)-\text{Proj}_{\Bc_s}(\theta_s(t+1/2))\right)^{\top} \left(\text{Proj}_{\Bc_s}(\theta_s(t+1/2))-\theta_s(t)\right)\geq 0
\end{align}
Here we write $\theta_s(t)+\eta\widehat{g}_s(\theta(t))$ as $\theta_s(t+1/2)$ to simplify the notation. The last inequality comes from the property of the projection onto a convex set.

Therefore, combining \eqref{eqn:actor_ana_2}, \eqref{eqn:actor_ana_3}, \eqref{eqn:actor_ana_4} and \eqref{eqn:actor_ana_5} suggests
\begin{equation*}
    \nabla_{\theta_s} J\left(\theta(t)\right)^{\top} \delta_s(t)\geq-\frac{4\tau DRd_{\zeta_s}^{1/2}}{(1-\gamma)\eta}\sum_{s\in\Sc}\left[\left\|Q_s(\,\cdot\,;\bar{\omega}_s,t) - Q_s^{\theta(t)}\left(\cdot\right)\right\|_{L^2(\nu_{\theta(t)})}\right] + \frac{1}{2}\left(\|\delta_{s}(t)\|_2^2 - \|\xi_s(t)\|_2^2\right).
\end{equation*}
Consequently,
\begin{equation}\label{eqn:actor_ana_6}
    \nabla_{\theta} J\left(\theta(t)\right)^{\top} \delta(t)\geq-\frac{4\tau DRd_{\zeta_s}^{1/2}}{(1-\gamma)\eta}|\Sc|\sum_{s\in\Sc}\left[\left\|Q_s(\,\cdot\,;\bar{\omega}_s,t) - Q_s^{\Pi^{\theta(t)}}\left(\cdot\right)\right\|_{L^2(\nu_{\theta(t)})}\right] + \frac{1}{2}\left(\|\delta(t)\|_2^2 - \|\xi(t)\|_2^2\right).
\end{equation}
Thus, by plugging \eqref{eqn:actor_ana_6} into \eqref{eqn:actor_ana_1} and by Assumption \ref{ass:variance_bound}, we have
\begin{align}\label{eqn:actor_ana_7}
    \frac{1-L\cdot\eta}{2}\E\left[\|\delta(t)\|_2^2\right]\leq\,&\eta^{-1}\cdot\E\left[J(\theta(t+1))-J(\theta(t))\right]+\frac{\tau^2\Sigma^2|\Sc|}{2B}\notag\\
    &+\frac{4\tau DRd_{\zeta_s}^{1/2}|\Sc|}{(1-\gamma)\eta}\sum_{s\in\Sc}\left\|Q_s(\,\cdot\,;\bar{\omega}_s,t) - Q_s^{\Pi^{\theta(t)}}\left(\cdot\right)\right\|_{L^2(\nu_{\theta(t)})}.
\end{align}
Here the expectation is taken over $\sigma_{\theta(t)}$ given $\{\bar{\omega}_s\}_{s\in\Sc}$.

Now, in order to bridge the gap between $\|\delta(t)\|_2$ in \eqref{eqn:actor_ana_7} and $\|\rho(t)\|_2=\|G(\theta(t))\|_2$ in \eqref{eqn:grad_map}, we next will bound the difference $\|\delta(t)-\rho(t)\|_2$.
We start with defining a local gradient mapping $G_s$ from $\R^{M\times d_{\zeta}}$ to $\R^{M\times d_{\zeta_s}}$:
\begin{equation}\label{eqn:grad_map_local}
    G_s(\theta):=\eta^{-1}\cdot\left[\text{Proj}_{\Bc_s}\left(\theta_s+\eta\cdot\nabla_{\theta_s} J(\theta)\right)-\theta_s\right].
\end{equation}
Since $\Bc_s$ is an $l_\infty$-ball around the initialization, it is easy to verify that $G_s(\theta)=(G(\theta))_s$.
Therefore, we can further define $\rho_s(t)=G_s(\theta(t))$ and the  following decomposition holds:
$$\|\delta(t)-\rho(t)\|_2^2=\sum_{s\in\Sc}\|\delta_s(t)-\rho_s(t)\|_2^2.$$
From the definitions of $\delta_s(t)$ and $\rho_s(t)$,
\begin{align*}
    \|\delta_s(t)-\rho_s(t)\|_2&=\eta^{-1}\cdot\left\|\text{Proj}_{\Bc_s}\left(\theta_s+\eta\cdot\nabla_{\theta_s} J(\theta)\right)-\theta_s-\text{Proj}_{\Bc_s}\left(\theta_s+\eta\cdot \widehat{g}_s(\theta)\right)+\theta_s\right\|_2\\
    &=\eta^{-1}\cdot\left\|\text{Proj}_{\Bc_s}\left(\theta_s+\eta\cdot\nabla_{\theta_s} J(\theta)\right)-\text{Proj}_{\Bc_s}\left(\theta_s+\eta\cdot \widehat{g}_s(\theta)\right)\right\|_2\\
    &\leq\eta^{-1}\cdot\left\|\theta_s+\eta\cdot\nabla_{\theta_s} J(\theta)-\theta_s+\eta\cdot \widehat{g}_s(\theta)\right\|_2=\left\|\nabla_{\theta_s} J(\theta)- \widehat{g}_s(\theta)\right\|_2
\end{align*}
Following similar calculations in  \cite[Lemma D.3]{cai2019neural},
\begin{align}\label{eqn:actor_ana_8}
    \E\left[\|\nabla_{\theta_s} J(\theta)- \widehat{g}_s(\theta)\|_2^2\right]&\leq\frac{2\tau^2\Sigma^2}{B}+\frac{8\tau^2D^2}{(1-\gamma)^2}\left(\sum_{s\in\Sc}\left\|Q_s(\,\cdot\,;\bar{\omega}_s,t) - Q_s^{\Pi^{\theta(t)}}\left(\cdot\right)\right\|_{L^2(\nu_{\theta(t)})}\right)^2\notag\\
    &\leq\frac{2\tau^2\Sigma^2}{B}+\frac{8\tau^2D^2|\Sc|}{(1-\gamma)^2}\left(\sum_{s\in\Sc}\left\|Q_s(\,\cdot\,;\bar{\omega}_s,t) - Q_s^{\Pi^{\theta(t)}}\left(\cdot\right)\right\|^2_{L^2(\nu_{\theta(t)})}\right).
\end{align}
The expectation is taken over $\sigma_{\theta(t)}$ given $\{\bar{\omega}_s\}_{s\in\Sc}$. Consequently,
\begin{align}\label{eqn:actor_ana_9}
    \E\left[\|\delta(t)- \rho(t)\|_2^2\right]\leq\frac{2\tau^2\Sigma^2|\Sc|}{B}+\frac{8\tau^2D^2|\Sc|^2}{(1-\gamma)^2}\left(\sum_{s\in\Sc}\left\|Q_s(\,\cdot\,;\bar{\omega}_s,t) - Q_s^{\Pi^{\theta(t)}}\left(\cdot\right)\right\|^2_{L^2(\nu_{\theta(t)})}\right).
\end{align}

Set $\eta=1/\sqrt{T_{\text{actor}}}$ and take \eqref{eqn:actor_ana_7} and \eqref{eqn:actor_ana_9}, we obtain \eqref{eqn:actor_stat} from the following estimations:
\begin{align*}
    \min_{t\in[T_{\text{actor}}]}\E\left[\|\rho(t)\|^2_2\right]&\leq\frac{1}{T_{\text{actor}}}\cdot\sum_{t=1}^{T_{\text{actor}}}\|\rho(t)\|^2_2\leq\frac{2}{T_{\text{actor}}}\cdot\sum_{t=1}^{T_{\text{actor}}}\left(\E\left[\|\delta(t)-\rho(t)\|^2_2\right] + \E\left[\|\delta(t)\|^2_2\right]\right)\notag\\
    &\leq\frac{2}{T_{\text{actor}}}\cdot\sum_{t=1}^{T_{\text{actor}}}\left(\E\left[\|\delta(t)-\rho(t)\|^2_2\right] + 2(1-L\cdot\eta)\E\left[\|\delta(t)\|^2_2\right]\right)\notag\\
    &\leq\frac{8\tau^2\Sigma^2|\Sc|}{B}+\frac{4}{\sqrt{T_{\text{actor}}}}\E[J(\theta({T_{\text{actor}}}+1))-J(\theta(1))]+\epsilon_Q(T_{\text{actor}}),
\end{align*}
where $\epsilon_Q$ measures the error accumulated from the critic steps which  is defined in \eqref{eqn:actor_stat_Q}, i.e.,
\begin{align*}
    \epsilon_Q(T_{\text{actor}})=&\,\frac{32\tau DRd_{\zeta_s}^{1/2}|\Sc|}{(1-\gamma)\eta T_{\text{actor}}}\cdot\sum_{t=1}^{T_{\text{actor}}} \sum_{s\in\Sc}\E\left[\left\|Q_s(\,\cdot\,;\bar{\omega}_s) - Q_s^{\Pi^{\theta(t)}}\left(\cdot\right)\right\|_{L^2(\nu_{\theta(t)})}\right]\notag\\
    &+\frac{16\tau^2D^2|\Sc|^2}{(1-\gamma)^2 T_{\text{actor}}}\cdot\sum_{t=1}^{T_{\text{actor}}}\sum_{s\in\Sc}\E\left[\left\|Q_s(\,\cdot\,;\bar{\omega}_s) - Q_s^{\Pi^{\theta(t)}}\left(\cdot\right)\right\|^2_{L^2(\nu_{\theta(t)})}\right].
\end{align*}
Here the expectations in \eqref{eqn:actor_stat} and \eqref{eqn:actor_stat_Q} are taken over all randomness in Algorithm \ref{DNTD} and Algorithm \ref{DNPG}.
\end{proof}
\hfill \mbox{\bf Q.E.D.}

\subsubsection{Bridging the gap between Stationarity and Optimality}\label{subsec:actor_2}

Recall that $\sigma_\theta$ in \eqref{eqn:visitation_measure} denotes the state-action visitation measure under policy $\Pi^\theta$. Denote $\bar{\sigma}_\theta$ as the state visitation measure under policy $\Pi^\theta$. Consequently, $$\bar{\sigma}_\theta(\mu)\Pi^\theta(h\mid\mu)={\sigma}_\theta(\mu,h).$$

Following similar steps in the proof of  \cite[Theorem 4.8]{cai2019neural}, one can characterize the global optimality of the obtained stationary point $\widetilde{\theta}\in\Bc$ as the following.

\begin{Lemma}\label{lemma:stat_opt}
Let $\widetilde{\theta}\in\Bc$ be a stationary point of $J(\cdot)$ satisfying  condition \eqref{eqn:stationary_theta} and let $\theta^*\in\Bc$ be the global maximum point of $J(\cdot)$ in $\Bc$. Then the following inequality holds:
\begin{align}\label{eqn:stat_opt_center}
    (1-\gamma)\left(J(\theta^*)-J(\widetilde\theta)\right)&\leq\frac{2{r}_\text{max}}{1-\gamma}\inf_{\theta\in\Bc} \left\| u_{\widetilde{\theta}}(\mu,h)- \sum_{s \in \Sc}\phi_{\widetilde{\theta}_s}(\mu(s),h(s))^\top\theta_s\right\|_{L^2(\sigma_{\widetilde\theta})},
\end{align}
where $ u_{\widetilde{\theta}}(\mu,h) := \frac{d\sigma_{\theta^*}}{d\sigma_{\widetilde\theta}}(\mu,h)-\frac{d\bar\sigma_{\theta^*}}{d\bar\sigma_{\widetilde\theta}}(\mu)+\sum_{s\in\Sc}\phi_{\widetilde{\theta}_s}(\mu(s),h(s))^\top\widetilde{\theta}_s$, and $\frac{\mathrm{d}\sigma_{\theta^*}}{\mathrm{d}\sigma_{\widetilde\theta}}$,$\frac{\mathrm{d}\bar\sigma_{\theta^*}}{\mathrm{d}\bar\sigma_{\widetilde\theta}}$ are the Radon-Nikodym derivatives between the corresponding measures.
\end{Lemma}

\begin{proof}{Proof of Lemma \ref{lemma:stat_opt}}
First recall that by \eqref{eqn:grad_nn}, for any $\theta\in\Bc$,
\begin{equation*}
    \nabla_{\theta} J(\widetilde{\theta})^{\top}(\theta-\widetilde{\theta}) = \sum_{s\in\Sc}\nabla_{\theta_s} J({\widetilde\theta})^{\top}(\theta_s-\widetilde{\theta}_s)=\frac{\tau}{1-\gamma}\sum_{s\in\Sc} \mathbb{E}_{\sigma_{\widetilde\theta}}\left[Q^{\Pi^{\widetilde\theta}}(\mu, h) \cdot \Phi(\widetilde\theta, s, \mu, h)^{\top}(\theta_s-\widetilde{\theta}_s)\right],
\end{equation*}
in which $\Phi(\theta, s, \mu, h):=\phi_{\theta_s}(\mu(s), h(s))-\mathbb{E}_{{h(s)' \sim \Pi_s^{\theta_s}(\cdot \mid \mu(s))}}\left[\phi_{\theta_s}\left(\mu(s), h'(s)\right)\right]$ is defined in \eqref{eqn:log_pi}.

Since $\widetilde{\theta}\in\Bc$ is a stationary point
of $J(\cdot)$, 
\begin{equation}\label{eqn:actor_key}
    \sum_{s\in\Sc} \mathbb{E}_{\sigma_{\widetilde\theta}}\left[Q^{\Pi^{\widetilde\theta}}(\mu, h) \cdot \Phi(\widetilde\theta, s, \mu, h)^{\top}(\theta_s-\widetilde{\theta}_s)\right]\leq 0, \quad \forall \theta \in \mathcal{B}.
\end{equation}
Denote $A^{\Pi^{\widetilde\theta}}(\mu,h):=Q^{\Pi^{\widetilde\theta}}(\mu, h) - V^{\Pi^{\widetilde\theta}}(\mu)$ as the advantage function under policy $\Pi^{\widetilde\theta}$. It holds from the definition that $\E_{h\sim\Pi^{\widetilde\theta}(\cdot\mid\mu)}[A^{\Pi^{\widetilde\theta}}(\mu,h)]=V^{\Pi^{\widetilde\theta}}(\mu)-V^{\Pi^{\widetilde\theta}}(\mu)=0$. Meanwhile,  $\sup_{(\mu,h)\in\Xi}\left|A^{\Pi^{\widetilde\theta}}(\mu,h)\right|\leq 2\sup_{\mu\in\Pc^N(\Sc)}\left|V^{\Pi^{\widetilde\theta}}(\mu)\right|\leq\frac{2r_\text{max}}{1-\gamma}$.

Given that $\E_{h\sim\Pi^{\widetilde\theta}(\cdot\mid\mu)}[A^{\Pi^{\widetilde\theta}}(\mu,h)]=0$ and $\E_{h\sim\Pi^{\widetilde\theta}(\cdot\mid\mu)}[\Phi(\widetilde\theta, s, \mu, h)]=0$, we have for any $s\in\Sc$,
\begin{eqnarray}\label{eqn:actor_meanfree1}
\mathbb{E}_{\sigma_{\widetilde\theta}}\left[V^{\Pi^{\widetilde\theta}}(\mu) \cdot \Phi(\widetilde\theta, s, \mu, h)\right]=0, \qquad \text{ and }
\end{eqnarray}
\begin{eqnarray}\label{eqn:actor_meanfree2}
\quad \mathbb{E}_{\sigma_{\widetilde\theta}}\left[A^{\Pi^{\widetilde\theta}}(\mu,h) \cdot \mathbb{E}_{{h(s)' \sim \Pi_s^{\widetilde\theta_s}(\cdot \mid \mu(s))}}\left[\phi_{\widetilde\theta_s}\left(\mu(s), h'(s)\right)\right]\right]=0.
\end{eqnarray}
Combining \eqref{eqn:actor_key} with \eqref{eqn:actor_meanfree1} and \eqref{eqn:actor_meanfree2},
\begin{equation}\label{eqn:actor_meanfree3}
    \sum_{s\in\Sc} \mathbb{E}_{\sigma_{\widetilde\theta}}\left[A^{\Pi^{\widetilde\theta}}(\mu, h) \cdot \phi_{\widetilde\theta_s}\left(\mu(s), h(s)\right)^{\top}(\theta_s-\widetilde{\theta}_s)\right]\leq 0, \quad \forall \theta \in \mathcal{B}.
\end{equation}

Moreover, by the Performance Difference Lemma (\citet{KL2002}),
\begin{equation}\label{eqn:perform_diff}
(1-\gamma) \cdot\left(J(\theta^*)-J(\widehat{\theta})\right)=\mathbb{E}_{\bar\sigma_{\theta^*}}\left[\left\langle A^{\Pi^{\widetilde\theta}}(\mu, \cdot), {\Pi^{\theta^*}}(\cdot \mid \mu)-\Pi^{\widetilde{\theta}}(\cdot \mid \mu)\right\rangle\right].
\end{equation}
Combining \eqref{eqn:perform_diff} with \eqref{eqn:actor_meanfree3}, it holds that for any $\theta \in \mathcal{B}$,
\begin{align}
&(1-\gamma) \cdot\left(J(\theta^*)-J(\widehat{\theta})\right)\nonumber\\
\leq&\,\mathbb{E}_{\bar\sigma_{\theta^*}}\left[\left\langle A^{\Pi^{\widetilde\theta}}(\mu, \cdot), {\Pi^{\theta^*}}(\cdot \mid \mu)-\Pi^{\widetilde{\theta}}(\cdot \mid \mu)\right\rangle\right] - \sum_{s\in\Sc} \mathbb{E}_{\sigma_{\widetilde\theta}}\left[A^{\Pi^{\widetilde\theta}}(\zeta) \cdot \phi_{\widetilde\theta_s}\left(\zeta_s\right)^{\top}(\theta_s-\widetilde{\theta}_s)\right]\nonumber\\
=&\,\mathbb{E}_{\sigma_{\widetilde\theta}}\left[A^{\Pi^{\widetilde\theta}}(\mu,h)\cdot\left(\frac{\mathrm{d}\sigma_{\theta^*}}{\mathrm{d}\sigma_{\widetilde\theta}}(\mu,h)-\frac{\mathrm{d}\bar\sigma_{\theta^*}}{\mathrm{d}\bar\sigma_{\widetilde\theta}}(\mu)-\sum_{s\in\Sc}\phi_{\widetilde{\theta}_s}(\mu(s),h(s))^\top(\theta_s-\widetilde{\theta}_s)\right)\right].
\end{align}
Therefore, 
\begin{align}
&(1-\gamma)\cdot\left(J(\theta^*)-J(\widehat{\theta})\right)\nonumber \\
\leq&\,\frac{2{r}_\text{max}}{1-\gamma}\inf_{\theta\in\Bc}\left\|\frac{\mathrm{d}\sigma_{\theta^*}}{\mathrm{d}\sigma_{\widetilde\theta}}(\mu,h)-\frac{\mathrm{d}\bar\sigma_{\theta^*}}{\mathrm{d}\bar\sigma_{\widetilde\theta}}(\mu)-\sum_{s\in\Sc}\phi_{\widetilde{\theta}_s}(\mu(s),h(s))^\top(\theta_s-\widetilde{\theta}_s)\right\|_{L^2(\sigma_{\widetilde\theta})} \nonumber\\
=&\,\frac{2{r}_\text{max}}{1-\gamma}\inf_{\theta\in\Bc} \left\| u_{\widetilde{\theta}}(\mu,h)- \sum_{s \in \Sc}\phi_{\widetilde{\theta}_s}(\mu(s),h(s))^\top\theta_s\right\|_{L^2(\sigma_{\widetilde\theta})},
\end{align}
where $ u_{\widetilde{\theta}}(\mu,h) := \frac{d\sigma_{\theta^*}}{d\sigma_{\widetilde\theta}}(\mu,h)-\frac{d\bar\sigma_{\theta^*}}{d\bar\sigma_{\widetilde\theta}}(\mu)+\sum_{s\in\Sc}\phi_{\widetilde{\theta}_s}(\mu(s),h(s))^\top\widetilde{\theta}_s$, and $\frac{\mathrm{d}\sigma_{\theta^*}}{\mathrm{d}\sigma_{\widetilde\theta}}$,$\frac{\mathrm{d}\bar\sigma_{\theta^*}}{\mathrm{d}\bar\sigma_{\widetilde\theta}}$ are the Radon-Nikodym derivatives between corresponding measures.
\end{proof}
\hfill \mbox{\bf Q.E.D.}

To further bound the right-hand-side of \eqref{eqn:stat_opt_center} in Lemma \ref{lemma:stat_opt}, define the following function class
\begin{align}\label{eqn:local_nn_policy}
    \widetilde{\mathcal{F}}_{R,M}=\bigg\{f_0(\zeta;\theta)&:= \sum_{s \in \Sc}\underbrace{\Bigg[\frac{1}{\sqrt{M}}\sum_{m=1}^{M}  \mathds{1}\left\{[\theta_{s}(0)]_m^{\top} \zeta_s>0\right\} [\theta_{s}]_m^{\top} \zeta_s\Bigg]}_{(\star)}:\notag\\
    &\theta_s \in \mathbb{R}^{M \times d_{\zeta_s}}, \|\theta_s-\theta_s(0)\|_{\infty} \leq R/\sqrt{M}\bigg\},
\end{align}
given an initialization $\theta_s(0)\in\R^{M\times d_{\zeta_s}}$, $s \in \Sc$ and $b\in\R^M$.

$\widetilde{\mathcal{F}}_{R,M}$ \eqref{eqn:local_nn_policy} is a local linearization of the actor neural network. 
More specifically, term ($\star$) in \eqref{eqn:local_nn_policy} locally linearizes the decentralized actor neural network $f(\zeta_s;\theta_s)$ \eqref{eqn:energy_policy} with respect to $\theta_s$. 
Any $f_0(\zeta;\theta)\in\widetilde{\mathcal{F}}_{R,M}$ is a sum of $|\Sc|$ inner products between feature mapping $\phi_{\theta_s(0)}(\cdot)$ \eqref{eqn:feature} and parameter $\theta_s$: $f_0(\zeta;\theta)=\sum_{s\in\Sc}\phi_{\theta_s(0)}(\zeta_s)\cdot\theta_s$. 
As the width of the neural network $M\to\infty$, $\widetilde\Fc_{R,M}$ converges to $\Fc_{R,\infty}$ (defined in Assumption \ref{ass:F_infty_policy}).
The approximation error between $\widetilde{\mathcal{F}}_{R,M}$ and $\Fc_{R,\infty}$ is bounded in the following lemma. 

\begin{Lemma}\label{lemma:F_infty_policy}
    For any function $f(\zeta)\in\Fc_{R,\infty}$ defined in Assumption \ref{ass:F_infty_policy}, we have
    \begin{equation}\label{eqn:F_infty_policy}
        \E_{{\rm{init}}}\left[\left\|f(\cdot) - \text{Proj}_{\widetilde\Fc_{R,M}}f(\cdot)\right\|_{L^2(\sigma_{\widetilde\theta})}\right]\leq\Oc\left(\frac{|\Sc|Rd^{1/2}_{\zeta_s}}{M^{1/2}}\right).
    \end{equation}
\end{Lemma}

Lemma \ref{lemma:F_infty_policy} follows from \citet{rahimi2008uniform} and \cite[Proposition 4.3]{cai2019neural}. The factor $|\Sc|$ stems from the fact that $\Fc_{R,\infty}$ can be decomposed into $|\Sc|$ independent reproducing kernel Hilbert spaces. 
With Lemma
\ref{lemma:F_infty_policy}, we are ready to establish an upper bound for the right-hand-side of \eqref{eqn:stat_opt_center} in the following proposition.

\begin{Proposition}\label{prop:opt_stat}
Under Assumption \ref{ass:F_infty_policy}, let $\widetilde{\theta}\in\Bc$ be a stationary point of $J(\cdot)$ and let $\theta^*\in\Bc$ be the global maximum point of $J(\cdot)$ in $\Bc$. Then the following inequality holds:
\begin{align}\label{eqn:stat_opt_bound}
    (1-\gamma)\left(J(\theta^*)-J(\widetilde\theta)\right)&\leq\Oc\left(\frac{|\Sc|R^{3/2}d^{3/4}_{\zeta_s}}{M^{1/4}}\right).
\end{align}
\end{Proposition}
\begin{proof}{Proof of Proposition \ref{prop:opt_stat}}
First by the triangle inequality,
\begin{align}\label{eqn:global_stationary_2}
\inf_{\theta\in\Bc} \left\| u_{\widetilde{\theta}}(\zeta) - \sum_{s\in\Sc}\phi_{\widetilde{\theta}_s}(\zeta_s)^\top\theta_s\right\|_{L^2(\sigma_{\widetilde\theta})}
&\leq\left\| u_{\widetilde{\theta}}(\zeta) - \text{Proj}_{{\widetilde\Fc}_{R,M}} u_{\widetilde{\theta}}(\zeta) \right\|_{L^2(\sigma_{\widetilde\theta})} \\
& + \inf_{\theta\in\Bc} \left\| \text{Proj}_{{\widetilde\Fc}_{R,M}} u_{\widetilde{\theta}}(\zeta)-\sum_{s\in\Sc}\phi_{\widetilde{\theta}_s}(\zeta_s)^\top\theta_s\right\|_{L^2(\sigma_{\widetilde\theta})}, \nonumber
\end{align}
where ${\widetilde\Fc}_{R,M}$ is defined in \eqref{eqn:local_nn_policy}. We denote $\text{Proj}_{{\widetilde\Fc}_{R,M}} u_{\widetilde{\theta}}(\zeta)=\sum_{s\in\Sc}\phi_{\theta_s(0)}(\zeta_s)\cdot\widehat\theta_s\in{\widetilde\Fc}_{R,M}$ for some $\widehat\theta\in\Bc$. Therefore, by Lemma  \ref{lemma:F_infty_policy},
the first term on the right-hand-side of \eqref{eqn:global_stationary_2} is bounded by \eqref{eqn:F_infty_policy}:
\begin{equation*}
    \left\| u_{\widetilde{\theta}}(\zeta) - \sum_{s\in\Sc}\phi_{\theta_s(0)}(\zeta_s)\cdot\widehat\theta_s \right\|_{L^2(\sigma_{\widetilde\theta})}\leq\Oc\left(\frac{|\Sc|Rd^{1/2}_{\zeta_s}}{M^{1/2}}\right).
\end{equation*}
The following Lemma \ref{lemma:feature_approximation} is a direct application of  \cite[Lemma E.2]{wang2019neural}, which is used to bound the second term on the right-hand-side of \eqref{eqn:global_stationary_2}.
\begin{Lemma}\label{lemma:feature_approximation}
It holds for any $\theta_s, \theta_s' \in \Bc_s=\big\{\alpha_s\in\R^{M\times d_{\zeta_s}}:\|\alpha_s-\theta_s(0)\|_\infty\leq R/\sqrt{M}\big\}$ that
\begin{eqnarray}
\E_{{\rm{init}}}\left[\|\phi_{{\theta}_s}(\zeta_s)^\top\theta_s' - \phi_{{\theta}_s(0)}(\zeta_s)^\top\theta_s' \|_{L^2(\sigma_{\theta})}\right] \leq \Oc\left(\frac{R^{3/2}d^{3/4}_{\zeta_s}}{M^{1/4}}\right),
\end{eqnarray}
where the expectation is taken over random initialization.
\end{Lemma}
Taking $\theta=\widetilde\theta$ and $\theta'=\widehat\theta$ in Lemma \ref{lemma:feature_approximation} gives us
\begin{equation*}
    \sum_{s\in\Sc} \left\| \phi_{\theta_s(0)}(\zeta_s)\cdot\widehat\theta_s-\phi_{\widetilde{\theta}_s}(\zeta_s)^\top\widehat\theta_s\right\|_{L^2(\sigma_{\widetilde\theta})}\leq\Oc\left(\frac{|\Sc|R^{3/2}d^{3/4}_{\zeta_s}}{M^{1/4}}\right)
\end{equation*}
Therefore, by Lemma \ref{lemma:global_opt}, 
\begin{equation*}
    (1-\gamma)\left(J(\theta^*)-J(\widetilde\theta)\right)\leq\inf_{\theta\in\Bc} \left\| u_{\widetilde{\theta}}(\zeta) - \sum_{s\in\Sc}\phi_{\widetilde{\theta}_s}(\zeta_s)^\top\theta_s\right\|_{L^2(\sigma_{\widetilde\theta})}\leq\Oc\left(\frac{|\Sc|R^{3/2}d^{3/4}_{\zeta_s}}{M^{1/4}}\right).
\end{equation*}
\end{proof}
\hfill \mbox{\bf Q.E.D.}

Now we are ready to establish Theorem \ref{thm:actor_conv}.
\begin{proof}{Proof of Theorem \ref{thm:actor_conv}}
Following similar calculations as in \cite[ Section H.3]{wang2019neural}, we obtain that at iteration $t\in[T_\text{actor}]$,
\begin{align}\label{eqn:trick}
\nabla_\theta J(\theta(t))^\top (\theta -\theta(t)) \leq 2(R + \frac{\eta \cdot r_{\max}}{1-\gamma}) \cdot \|\rho(t)\|_2, \;\;\; \forall \theta \in \Bc.
\end{align}
The right-hand-side of \reff{eqn:trick} quantifies the deviation of $\theta(t)$ from a stationary point $\widetilde\theta$. Having \reff{eqn:trick} and following similar arguments  for Lemma \ref{lemma:stat_opt} and Proposition \ref{prop:opt_stat}, we can show that
\begin{align}\label{eqn:combine_key}
(1-\gamma) \min_{t\in[T_{\mathrm{actor}}]}\E\left[J(\theta^*)-J(\theta(t))\right] \leq \Oc\left(\frac{|\Sc|R^{3/2}d^{3/4}_{\zeta_s}}{M^{1/4}}\right) + 2\left(R + \frac{\eta \cdot r_{\max}}{1-\gamma}\right) \cdot  \min_{t\in[T_{\mathrm{actor}}]}\E[\|\rho(t)\|_2].
\end{align}
Here the last term $\min_{t\in[T_{\mathrm{actor}}]}\E[\|\rho(t)\|_2]$ is bounded by \eqref{eqn:actor_stat} in Theorem \ref{thm:actor_stationary}, while the term $\epsilon_Q(T_{\mathrm{actor}})$ in \eqref{eqn:actor_stat_Q} can be upper bounded by Theorem \ref{thm:critic_conv}. Finally with the parameters stated in Theorem \ref{thm:actor_conv},
the following statement holds by straightforward calculation:
\begin{equation*}
    \min_{t\in[T_{\mathrm{actor}}]}\E\left[J(\theta^*)-J(\theta(t))\right]\leq\Oc\left(|\Sc|^{1/2}B^{-1/2} + |\Sc||\Ac|^{1/4}\left(\gamma^{k/8}+(T_{\mathrm{actor}})^{-1/4}\right)\right).
\end{equation*}
\end{proof}
\hfill \mbox{\bf Q.E.D.}

\end{appendix}

\end{document}